\documentclass{article} %
\usepackage{iclr2026_conference,times}

\usepackage{hyperref}
\usepackage{url}

\usepackage[utf8]{inputenc} %

\usepackage{graphicx}

\usepackage[T1]{fontenc}

\IfFileExists{headers/config/showoverfull.config}{
	\overfullrule=1cm
}{
}

\usepackage{marginnote}

\usepackage[backgroundcolor=none,linecolor=red,textsize=footnotesize]{todonotes}

\usepackage{etoolbox}

\newbool{includeappendix}
\setbool{includeappendix}{true} %
\IfFileExists{headers/config/noappendix.config}{
	\setbool{includeappendix}{false}
}{}

\newif\ifincludeappendixx
\ifbool{includeappendix}{
	\includeappendixxtrue
}{
	\includeappendixxfalse
}

\usepackage{xr} %
\usepackage{filecontents}

\ifbool{includeappendix}{}{
	\input{appendix-labels-loader}

	\externaldocument{appendix-labels}
}

\definecolor{my-full-blue}{HTML}{1F77B4}

\definecolor{my-full-orange}{HTML}{FF7F0E}

\definecolor{my-full-green}{HTML}{2CA02C}

\definecolor{my-full-red}{HTML}{d62728}

\definecolor{my-full-purple}{HTML}{9467bd}

\definecolor{my-full-brown}{HTML}{8c564b}

\definecolor{my-full-pink}{HTML}{e377c2}

\definecolor{my-full-gray}{HTML}{7f7f7f}

\definecolor{my-full-olive}{HTML}{bcbd22}

\definecolor{my-full-cyan}{HTML}{17becf}

\definecolor{muted-green}{RGB}{87,132,89}
\definecolor{muted-red}{RGB}{218,90,84}

\colorlet{cgreen}{muted-green}
\colorlet{cred}{muted-red}

\colorlet{my-blue}{my-full-blue!30}
\colorlet{my-orange}{my-full-orange!30}
\colorlet{my-green}{my-full-green!30}
\colorlet{my-red}{my-full-red!30}
\colorlet{my-purple}{my-full-purple!30}
\colorlet{my-brown}{my-full-brown!30}
\colorlet{my-pink}{my-full-pink!30}
\colorlet{my-gray}{my-full-gray!30}
\colorlet{my-olive}{my-full-olive!30}
\colorlet{my-cyan}{my-full-cyan!30}

\usepackage{listings}

\usepackage{textcomp}

\usepackage{xcolor}

\usepackage[scaled=0.8]{beramono}

\definecolor{ckeyword}{HTML}{7F0055}
\definecolor{ccomment}{HTML}{3F7F5F}
\definecolor{cstring}{HTML}{2A0099}

\lstdefinestyle{numbers}{
	numbers=left,
	framexleftmargin=20pt,
	numberstyle=\tiny,
	firstnumber=auto,
	numbersep=1em,
	xleftmargin=2em
}

\lstdefinestyle{layout}{
	frame=none,
	captionpos=b,
}

\lstdefinestyle{comment-style}{
	morecomment=[l]//,
	morecomment=[s]{/*}{*/},
	commentstyle={\color{ccomment}\itshape},
}

\lstdefinestyle{string-style}{
	morestring=[b]",%
	morestring=[b]',%
	stringstyle={\color{cstring}},
	showstringspaces=false,%
}

\lstdefinestyle{keyword-style}{
	keywordstyle={\ttfamily\bfseries},
	morekeywords={
		function,
		constructor,
		int,
		bool,
		return,
		returns,
		uint
	},
	morekeywords = [2]{},
	keywordstyle = [2]{\text},
	sensitive=true,
}

\lstdefinestyle{input-encoding}{
	inputencoding=utf8,
	extendedchars=true,
	literate=
	{ℝ}{$\reals$}1%
	{→}{$\rightarrow$}1%
	{α}{$\alpha$}1%
	{β}{$\beta$}1%
	{λ}{$\lambda$}1%
	{θ}{$\theta$}1%
	{ϕ}{$\phi$}1%
}

\lstdefinestyle{escaping}{
	moredelim={**[is][\color{blue}]{\%}{\%}},
	escapechar=|,
	mathescape=true
}

\lstdefinestyle{default-style}{
	basicstyle=\fontencoding{T1}\ttfamily\footnotesize,
	style=numbers,
	style=layout,
	style=comment-style,
	style=string-style,
	style=keyword-style,
	style=input-encoding,
	style=escaping,
	tabsize=2,
	upquote=true
}

\lstdefinelanguage{BASIC}{
	language=C++,
	style=default-style
}[keywords,comments,strings]%

\usepackage{amsmath,amsfonts,bm}

\def\eqref#1{equation~\ref{#1}}

\def\1{\bm{1}}

\def\vx{{\bm{x}}}

\DeclareMathAlphabet{\mathsfit}{\encodingdefault}{\sfdefault}{m}{sl}
\SetMathAlphabet{\mathsfit}{bold}{\encodingdefault}{\sfdefault}{bx}{n}

\def\gH{{\mathcal{H}}}

\def\gN{{\mathcal{N}}}

\def\gX{{\mathcal{X}}}
\def\gY{{\mathcal{Y}}}

\def\sB{{\mathbb{B}}}

\newcommand{\E}{\mathbb{E}}
\newcommand{\Prob}{\mathbb{P}}

\newcommand{\R}{\mathbb{R}}

\DeclareMathOperator*{\argmax}{arg\,max}

\usepackage{adjustbox}
\usepackage[capitalise]{cleveref}
\usepackage{scrextend}

\usepackage{enumitem}
\usepackage{threeparttable}
\usepackage{multirow, makecell}
\usepackage{booktabs}
\usepackage{xspace}
\usepackage{wrapfig}
\usepackage{dsfont}

\usepackage{subcaption}
\usepackage{caption}
\usepackage{algorithm,algcompatible,amssymb,amsmath}

\algnewcommand\RETURN{\State \textbf{return} }
\usepackage{amsthm}
\usepackage{thmtools}
\usepackage{thm-restate}
\usepackage{mathtools}
\usepackage[makeroom]{cancel}

\declaretheoremstyle[
  spaceabove=0.5em, %
  spacebelow=0.01em,%
  headfont=\normalfont\bfseries,
  bodyfont=\normalfont,
  postheadspace=0.5em,
]{mystyle}

\declaretheoremstyle[
  spaceabove=0.5em, %
  spacebelow=0.01em,%
  headfont=\normalfont\itshape,
  bodyfont=\normalfont,
  postheadspace=0.5em,
  qed=$\square$,
]{myproofstyle}

\crefname{thm}{Theorem}{Theorems}

\crefname{lem}{Lemma}{Lemmas}

\crefname{cor}{Corollary}{Corollaries}

\crefname{defi}{Definition}{Definitions}

\usepackage{tikz}
\usetikzlibrary{calc,decorations,decorations.pathmorphing}
\usetikzlibrary{positioning,fit,arrows}
\usetikzlibrary{decorations.markings}
\usetikzlibrary{shapes,shapes.geometric}
\usetikzlibrary{shadows,patterns,snakes}
\usetikzlibrary{backgrounds,decorations.pathreplacing,automata}
\usetikzlibrary{intersections}
\usetikzlibrary{angles,quotes}
\usetikzlibrary{plotmarks}
\usetikzlibrary{patterns}
\usetikzlibrary{arrows.meta}
\usetikzlibrary{shapes.misc}
\usetikzlibrary{chains}
\usepackage{siunitx}
\newcolumntype{d}[1]{S[table-format=#1]}
\usetikzlibrary{patterns.meta}

\makeatletter
\def\extractcoord#1#2#3{
	\path let \p1=(#3) in \pgfextra{
		\pgfmathsetmacro#1{\x{1}/\pgf@xx}
		\pgfmathsetmacro#2{\y{1}/\pgf@yy}
		\xdef#1{#1} \xdef#2{#2}
	};
}
\makeatother

\usepackage{pifont}%

\usepackage{comment}

\newcommand{\cifar}{CIFAR-10\xspace}
\newcommand{\IN}{\textsc{ImageNet}\xspace}
\newcommand{\predict}{\textsc{Predict}\xspace}
\newcommand{\certify}{\textsc{Certify}\xspace}


\crefformat{section}{\S#2#1#3}

\crefrangeformat{section}{\S#3#1#4\crefrangeconjunction\S#5#2#6}

\crefmultiformat{section}{\S#2#1#3}{\crefpairconjunction\S#2#1#3}{\crefmiddleconjunction\S#2#1#3}{\creflastconjunction\S#2#1#3}

\newcommand{\crefrangeconjunction}{--}

\crefname{listing}{Lst.}{listings}
\crefname{line}{Lin.}{Lin.}
\crefname{appendix}{App.}{App.}

\newcommand{\appref}[1]{%
	\ifbool{includeappendix}{\cref{#1}}{the appendix}%
}
\newcommand{\Appref}[1]{%
	\ifbool{includeappendix}{\cref{#1}}{The appendix}%
}

\title{Dual Randomized Smoothing: Beyond Global Noise Variance}

\author{
Chenhao Sun, \, Yuhao Mao, \, Martin Vechev \\
\, Department of Computer Science, ETH Z\"urich, Switzerland \\
\, \{chenhao.sun, yuhao.mao, martin.vechev\}@inf.ethz.ch
}

\iclrfinalcopy %
\begin{document}

\maketitle
\begin{abstract}
   Randomized Smoothing (RS) is a prominent technique for certifying the robustness of neural networks against adversarial perturbations. With RS, achieving high accuracy at small radii requires a small noise variance, while achieving high accuracy at large radii requires a large noise variance. However, the global noise variance used in the standard RS formulation leads to a fundamental limitation: there exists no global noise variance that simultaneously achieves strong performance at both small and large radii. To break through the global variance limitation, we propose a dual RS framework which enables input-dependent noise variances. To achieve that, we first prove that RS remains valid with input-dependent noise variances, provided the variance is locally constant around each input. Building on this result, we introduce two components which form our dual RS framework: (i) a variance estimator first predicts an optimal noise variance for each input, (ii) this estimated variance is then used by a standard RS classifier. The variance estimator is independently smoothed via RS to ensure local constancy, enabling flexible design. We also introduce efficient training strategies to iteratively optimize the two components involved in the framework. Extensive experiments on the CIFAR-10 dataset demonstrate that our dual RS method provides strong performance for both small and large radii—unattainable with global noise variance—while incurring only a 60\% computational overhead at inference. Moreover, it consistently outperforms prior input-dependent noise approaches across most radii, with particularly large gains at radii 0.5, 0.75, and 1.0, achieving relative improvements of 15.6\%, 20.0\%, and 15.7\%, respectively. On \IN, dual RS remains effective across all radii, with 8.6\%, 17.1\% and 9.1\% performance advantages at radii 0.5, 1.0 and 1.5 respectively. Additionally, the proposed dual RS framework naturally provides a routing perspective for certified robustness, improving the accuracy-robustness trade-off with off-the-shelf expert RS models. Our code is available at \url{https://github.com/eth-sri/Dual-Randomized-Smoothing}.

\end{abstract}

\section{Introduction}

Deep neural networks have achieved remarkable success across diverse tasks but remain highly vulnerable to adversarial attacks; small, carefully crafted perturbations can lead to incorrect or unexpected predictions. This vulnerability has made adversarial robustness, which ensures consistent model outputs under small perturbations, a critical research focus. As heuristic defenses are often unreliable \citep{athalye2018obfuscated, Croce020a}, methods with provable robustness guarantees have become increasingly important.

Randomized Smoothing (RS) is a prominent technique for certifying robustness against $\ell_2$-norm adversarial perturbations. It constructs a smoothed classifier by adding Gaussian noise to the input and taking the majority vote of predictions, thereby ensuring consistent outputs within a certified neighborhood. Prior work has primarily focused on two directions: (1) training-based RS, which improves robustness by explicitly training the base classifier on noisy inputs \citep{CohenRK19, SalmanLRZZBY19, JeongS20, ZhaiDHZGRH020, JeongPKLKS21, JeongKS23}, and (2) denoised smoothing, where noisy inputs are first denoised before classification \citep{SalmanSYKK20, carlinicertified}. Recent advances in deep learning, particularly diffusion models, have significantly enhanced denoised smoothing approaches, enabling state-of-the-art certified accuracy at small perturbation radii \citep{carlinicertified, xiao2023densepure, zhang2023diffsmooth}.

\begin{figure*}[]
    \centering
    \begin{subfigure}[t]{.4\linewidth}
        \centering
        \includegraphics[width=.98\linewidth]{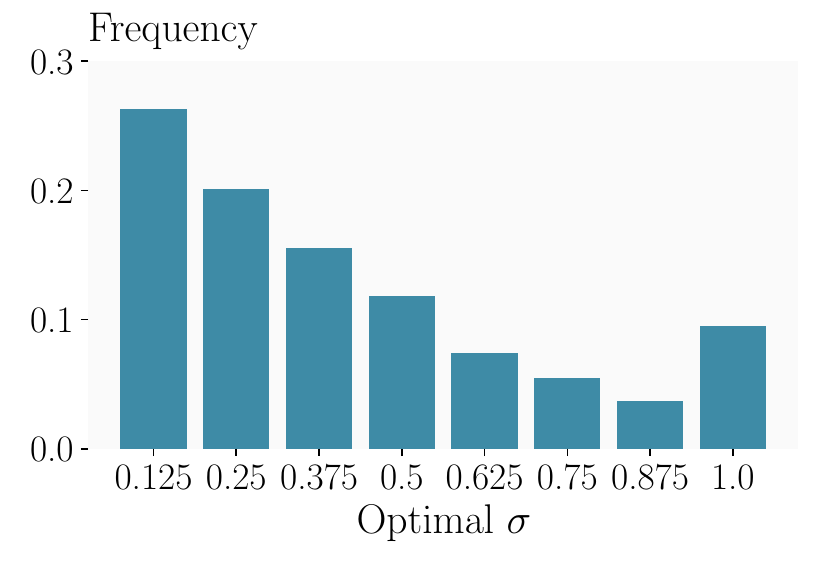}
    \end{subfigure}
    \begin{subfigure}[t]{.4\linewidth}
        \centering
        \includegraphics[width=.98\linewidth]{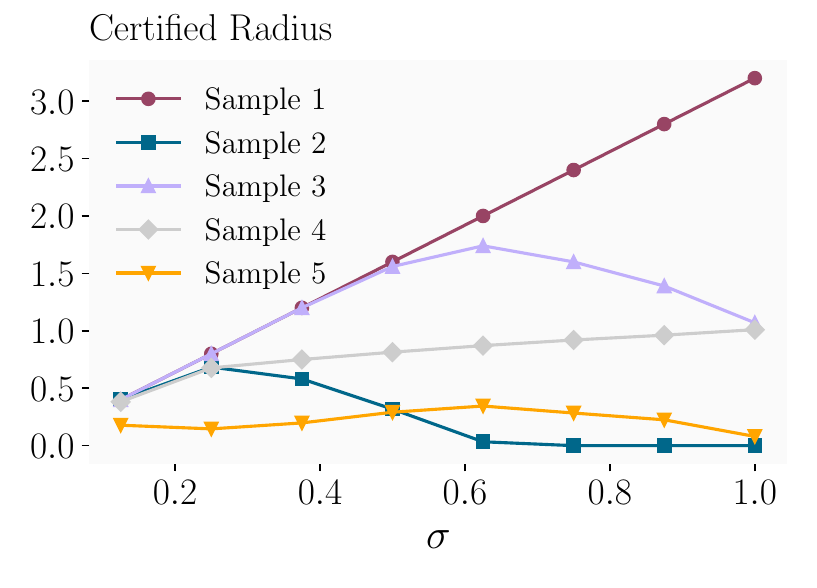}
    \end{subfigure}
    \vspace{-4mm}
    \caption{Left: The distribution of the optimal $\sigma$ on \cifar test set, where the base model is fixed to the pretrained denoised smoothing model from \citet{carlinicertified}. The optimal $\sigma$ for each input is defined as the $\sigma$ that maximizes the certified radius under the standard RS certification. Right: The certified radii curve of five independent samples against $\sigma$.}
    \label{fig:opt_dist}
    \vspace{-2mm}
\end{figure*}

\begin{table}[]
\caption{Comparison of key features of the literature and the proposed Dual RS.}
\vspace{-2mm}
\label{tab:related_work}
\centering
    \resizebox{.85\linewidth}{!}{
    \begin{tabular}{ccccc}
        \toprule
                                    & Literature & Flexible $\sigma$ & No test-time memorization & Flexible routing \\ \midrule
        Certified Routing            &  \citet{mueller2021boosting} & NA                &  \ding{51}              & Restricted             \\ \cmidrule{1-5}
        \multirow{5}{*}{Adaptive RS} & \citet{alfarra2022data}   & \ding{51}           &  \ding{55}                 &  NA         \\ 
                                    & \citet{wang2021pretrain} & \ding{51}           &  \ding{55}               & NA         \\ 
                                    & \citet{sukenik2022intriguing} & Restricted           &  \ding{51}              &  NA         \\ 
                                   &  \citet{jeong2024multi}          & Biased           &  \ding{51}              & NA  \\ 
                                  & This work & \ding{51}         &  \ding{51}             & \ding{51}          \\ \bottomrule
    \end{tabular}
    }
    \vspace{-5mm}
\end{table}

Despite recent advances, RS remains limited by a fundamental accuracy-robustness trade-off. Achieving a larger certified radius requires increasing the noise variance, which often reduces certified accuracy at smaller radii. This trade-off arises because prior methods apply a global noise variance shared across all inputs \citep{CohenRK19}. As illustrated in \cref{fig:opt_dist}, the noise variance that maximizes the certified radius varies substantially across samples. Recent work has explored input-dependent RS to mitigate this issue, but existing approaches either rely on test-time memorization \citep{alfarra2022data, wang2021pretrain}, intrinsically restrict adaptivity \citep{sukenik2022intriguing}, or systematically over-estimate the optimal variance \citep{jeong2024multi}. 

Motivated by these limitations, we propose \emph{Dual Randomized Smoothing} (Dual RS), a novel framework that enables RS certification with input-dependent noise variances. Our key insight is that RS certification remains valid, with appropriate confidence adjustments, as long as the noise variance is locally constant within the certified region rather than globally fixed across all inputs. 

\textbf{Main Contributions.} Our key contributions are:
\begin{itemize}
    \item A generalization of RS certification to locally constant noise variances, enabling flexible models to predict an optimal variance for each input. This generalization supports more favorable accuracy-robustness trade-offs, removing the fundamental limitation of global noise variance.
    \item A dual RS framework consisting of a variance estimator and a standard RS classifier. The variance estimator predicts the optimal $\sigma$ for each input, which is then used by the classifier for RS inference. We develop an iterative training procedure that sequentially optimizes both components. An alternative routing perspective is also discussed, where the variance estimator acts as a router that selects an appropriate off-the-shelf expert RS classifier based on the input. \cref{tab:related_work} compares key features of prior works with our proposed method.
    \item An extensive experimental evaluation of Dual RS, showing that Dual RS achieves strong performance across both small and large radii, outperforming prior input-dependent noise methods at most radii while adding roughly 60\% computational overhead at inference, compared to standard RS. Comparing against prior works, relative improvements of 15.6\%, 20.0\%, and 15.7\% are achieved at radii 0.5, 0.75, and 1.0 on \cifar, respectively, and 8.6\%, 17.1\% and 9.1\%  performance gain is delivered on \IN at radii 0.5, 1.0, and 1.5 respectively.
\end{itemize}

\section{Related Work} \label{sec:related_work}

\paragraph{Provable Adversarial Robustness}  
Empirical defenses against adversarial attacks are often unreliable \citep{athalye2018obfuscated, Croce020a}, motivating research on \emph{provable adversarial robustness}. Existing approaches fall into two categories: deterministic and probabilistic. Deterministic methods provide exact guarantees but scale poorly to large models \citep{gowal2018effectiveness, mirman2018differentiable, shi2021fast, muellercertified, de2024expressive, mao2023connecting, mao24understanding,baader2024expressivity, balauca2024overcoming, maoctbench,mao2026expressiveness}. Therefore, \emph{Randomized Smoothing} (RS) \citep{LecuyerAG0J19, CohenRK19} becomes the most widely used probabilistic method due to its scalability. Many works have improved RS by developing better training algorithms \citep{SalmanLRZZBY19, JeongS20, ZhaiDHZGRH020, JeongPKLKS21, JeongKS23}, leveraging pretrained models to construct base classifiers \citep{SalmanSYKK20, carlinicertified}, extending RS to different norms and noise distributions \citep{yang2020randomized, kumar2020curse}, designing alternative certification procedures \citep{xia2024mitigating, cullen2022double, li2022double}, proposing new evaluation metrics \citep{chenhao25acr}, and exploring ensemble techniques \citep{horvathboosting, liu2021enhancing}. However, a common limitation of these works is the use of a global noise variance in the smoothing distribution for all inputs, which leads to an inherent accuracy-robustness trade-off.

\paragraph{Input-dependent Randomized Smoothing}  
To mitigate the accuracy-robustness trade-off, recent works have explored adapting the noise variance per input. However, existing methods have notable limitations. Some rely on test-time memorization and are computationally expensive \citep{wang2021pretrain, alfarra2022data}. \citet{sukenik2022intriguing} provide theoretical guarantees for varying $\sigma$ with severely limited adaptivity. \citet{jeong2024multi} propose a multi-scale RS framework that cascades models with fixed variances, yet it always selects the largest variance that certifies an input, which often yields suboptimal results (\cref{fig:opt_dist}). Finally, \citet{lyu2024adaptive} introduce a two-stage framework for $\ell_\infty$ norm by splitting a fixed noise budget, but it lacks flexible per-input adaptiveness and fails to generalize to $\ell_2$ norms.

\section{Background}

This section introduces the key concepts of adversarial robustness and randomized smoothing (RS).

\textbf{Adversarial Robustness.}  
A model $f$ is adversarially robust if it produces consistent outputs under small perturbations. Given input $x$ and label $y$ with $f(x)=y$, $f$ is robust (with regard to $\ell_2$ norm) if $f(x')=f(x)$ for all $x'$ in $S(x)=\{x' \mid \|x'-x\|_2 \leq \epsilon\}$, where $\epsilon$ defines the perturbation magnitude.

\textbf{Randomized Smoothing.}  
RS provides certified robustness by constructing a smoothed classifier $g_c(x)=\argmax_{y \in \mathcal{Y}} \mathbb{P}_{\delta \sim \mathcal{N}(\mathbf{0},\sigma^2\mathbf{I})}[f(x+\delta)=y]$, where $f$ is the base classifier. The classifier $g_c$ is certifiably robust within an $\ell_2$ ball if the predicted class has probability greater than $0.5$ \citep{CohenRK19}. Improving the probability margin enhances the certified radius.

\textbf{Denoised Smoothing.}  
Denoised smoothing \citep{SalmanSYKK20} applies a denoiser before classification, i.e., $f(x+\delta)=f_{\mathrm{cls}}(\texttt{denoise}(x+\delta))$, where $\texttt{denoise}$ removes noise from the perturbed input and $f_{\mathrm{cls}}$ performs classification. This approach serves as a powerful paradigm for constructing RS base classifiers. Diffusion models have proven to be highly effective denoisers \citep{carlinicertified}, achieving state-of-the-art performance with off-the-shelf components. Following \citet{carlinicertified, jeong2024multi}, we adopt diffusion-based denoised smoothing to build base classifiers in our framework.

\section{Certification with Locally Constant Noise Variance} \label{sec:theory}

In this section, we formalize the main theoretical contribution of this work: we prove that RS certification remains valid when the noise variance is input-dependent, as long as it is constant within the certified region. This result provides the theoretical foundation for our dual RS framework.

Let $\gX \subseteq \R^d$ be the input space, $\gY$ the output space, and $f_c: \gX \rightarrow \gY$ the base classifier. The classifier smoothed with a Gaussian distribution $\gN(\bm{0}, \sigma^2 \bm{I})$ is defined as $g_c(\vx, \sigma) := \argmax_{y \in \gY} \Prob_{\bm{\delta} \sim \gN(\bm{0}, \sigma^2 \bm{I})}(f_c(\vx + \bm{\delta}) = y)$. Let $p_\sigma(\vx)$ be the probability of the most likely class, i.e., $p_\sigma(\vx):=\max_{y \in \gY} \Prob_{\bm{\delta} \sim \gN(\bm{0}, \sigma^2 \bm{I})}(f_c(\vx + \bm{\delta}) = y)$. \citet{CohenRK19} prove that with a global $\sigma$ constant in $\gX$, the smoothed classifier $g_c$ is certifiably robust within an $\ell_2$ ball $\sB(\vx, R(\vx, \sigma))$ of radius $R(\vx, \sigma) := \sigma \Phi^{-1}(p_\sigma(\vx))$ centered at the input $\vx$, where $\Phi$ is the cumulative distribution function of the standard Gaussian distribution. We replace the global $\sigma$ with an input-dependent function $\sigma: \gX \rightarrow \Sigma$, where $\Sigma \subset \R^+$ is the discrete set of allowed values, and denote the smoothed classifier with input-dependent noise variance as $g_c(\vx, \sigma(\vx))$. Building on this setup, we present a certification theorem that refines the result of \citet{CohenRK19} by relaxing the assumption on $\sigma$ from being globally constant to locally constant.

\begin{restatable}[Certification with Locally Constant $\sigma$]{thm}{adaptiveRSdeterministic} \label{thm:adaptive-RS-deterministic}
    Fix $\vx_0 \in \gX$ and $f_c$. Assume $\sigma(\vx)$ is constant within the $\ell_2$ ball $\sB(\vx_0, R_\sigma)$. Then for all $\vx$ such that $\|\vx - \vx_0\|_2 < \min(R_\sigma, R(\vx_0, \sigma(\vx_0)))$, we have $g_c(\vx, \sigma(\vx)) = g_c(\vx_0, \sigma(\vx_0))$.
\end{restatable}

The proof follows by carefully adapting the alternative argument of \citet{SalmanLRZZBY19} on the result of \citet{CohenRK19}, which leverages Lipschitz continuity, to remove the reliance on the global constancy of $\sigma$. The detailed proof is deferred to \cref{app:proof-of-adaptive-RS-deterministic}.

Practically, the assumption that $\sigma$ is constant within a neighborhood of $\vx_0$ can be satisfied in two ways: (1) by designing $\sigma(\vx)$ to be piecewise constant \citep{wang2021pretrain, alfarra2022data}, or (2) by certifying that $\sigma(\vx)$ is locally constant using deterministic certification methods \citep{SinghGPV19, WongK18, MullerMSPV22, ShiJKZJH24}.  
Approaches in the former category typically rely on test-time memorization, which is undesirable in practice. In contrast, approaches in the latter category, though extensively developed, are usually computationally expensive and less scalable.
Therefore, in this work, we seek a certification of $\sigma(\vx)$ that both scales well and eliminates test-time memorization. To this end, we propose to use a separate RS model to learn effective $\sigma(\vx)$ and certify the local constancy. To achieve this, we need to extend \cref{thm:adaptive-RS-deterministic} to a probabilistic setting, since RS in practice only provides probabilistic guarantees.

Before presenting the theorem, we extend the notion of RS to the practical setting, where $p_\sigma$ is lower bounded with uncertainty $\alpha$. Given $N$ trials of the event $I(f_c(\vx + \bm{\delta}) = y)$ and a predefined threshold $\alpha$, we can derive a lower bound $\hat{p}_\sigma$ such that $\Prob(p_\sigma \ge \hat{p}_\sigma) \ge 1-\alpha$ \citep{CohenRK19}. Consequently, the smoothed classifier $g_c(\vx, \sigma)$ is certifiably robust within the $\ell_2$ ball $\sB(\vx, \sigma \Phi^{-1}(\hat{p}_\sigma))$ with probability at least $1-\alpha$. Now we are ready to present the probabilistic version of \cref{thm:adaptive-RS-deterministic}.

\begin{restatable}[Probabilistic Guarantee with Confidence Adjustment]{thm}{adaptiveRSprobabilistic} \label{thm:adaptive-RS-probabilistic}
    Fix $\vx_0 \in \gX$ and $f_c$. Assume $g_c(\vx, \sigma(\vx_0))$ is certifiably robust within $\sB(\vx_0, R_c)$ with probability at least $1-\alpha$, and $\sigma(\vx)$ is constant within $\sB(\vx_0, R_\sigma)$ with probability at least $1-\beta$. Then for all $\vx$ such that $\|\vx - \vx_0\|_2 < \min(R_\sigma, R_c)$, we have $g_c(\vx, \sigma(\vx)) = g_c(\vx_0, \sigma(\vx_0))$ with probability at least $1-\alpha-\beta$.
\end{restatable}

The proof follows by applying union bound to upper bound the failure probability. The detailed proof is deferred to \cref{app:proof-of-adaptive-RS-probabilistic}. Note that \cref{thm:adaptive-RS-probabilistic} does not assume independence between the two failure events, and therefore remains valid even when the two failure events are correlated, e.g., correlated noise samples may be used in two certifications.

\textbf{Comparison with Prior Works.} Although not explicitly formalized, the idea of using a locally constant $\sigma$ has been explored in prior work \citep{wang2021pretrain, alfarra2022data}. \citet{wang2021pretrain} partition $\gX$ into a collection of $\ell_2$ balls, referred to as robust regions, and assign a constant $\sigma$ to each region. These regions are allocated and stored at test time, which prevents parallel inference and leads to dependence on the prior test cases. Similar strategies are adopted by \citet{alfarra2022data}. Beyond formalization and rigorous proof, \cref{thm:adaptive-RS-deterministic} further improves by eliminating the need for test-time memorization and instead ensuring local constancy through certifying $\sigma(\vx)$, which can be any learned model or hand-crafted function. 

Separately, \citet{sukenik2022intriguing} also study RS with input-dependent $\sigma(\vx)$ and show that proofs based on Neyman-Pearson lemma cannot allow reasonably flexible $\sigma(\vx)$. We circumvent this limitation by leveraging a proof based on Lipschitz continuity, similar to \citet{SalmanLRZZBY19,jeong2024multi}, which enables much greater flexibility in defining $\sigma(\vx)$. Note that our result does not restrict the behavior of $\sigma(\vx)$ outside the certified region, which can be arbitrarily complex. \cref{tab:related_work} summarizes the difference between this work and prior works.

Despite these advantages, \cref{thm:adaptive-RS-probabilistic} introduces a confidence penalty of $\beta$ to account for the probabilistic guarantee of $\sigma(\vx)$ being locally constant. This cost is inevitable when using any certification method that is not deterministic. However, in practice, we find that this cost is negligible when using RS to certify $\sigma(\vx)$. We list a few numerical examples under different configurations in \cref{tab:beta_cost} in \cref{app:deferred-tables-figures}, confirming that $\beta$ has minimal impact on the certified radius.

\section{The Dual Randomized Smoothing Framework}\label{sec:ars_frame}

\begin{figure}
    \centering
    \resizebox{.8\linewidth}{!}{
    \begin{tikzpicture}[>=Stealth, line width=.9pt]

        \node (x_1)         at (0, 3.0) {$\vx$};
        \node (sigma_est)   at (0, 4.5) {$g_e(\vx; \sigma_e)$};
        \begin{scope}[on background layer]
        \node (sigma_est_box) [fit=(sigma_est), 
            rounded corners,
            fill=red!10, inner sep=10pt] {};
        \end{scope}
        \draw[->, color=cgreen] (x_1) -- ($(sigma_est_box.south)$);

        \node (sigma_c)     at (2.6, 4.7) {$\sigma_c(x)$};
        \node (R_sigma)     at (2.4, 4.3) {$R_{\sigma}$};
        \draw[->, color=cgreen] ($(sigma_est_box.east)+(0,0.2)$) -- ($(sigma_c.west)$);
        \draw[->] ($(sigma_est_box.east)+(0,-0.2)$) -- ($(R_sigma.west)$);

        \node (cls)   at (5.2, 4.5) {$g_c(\vx; \sigma_c(\vx))$};
        \begin{scope}[on background layer]
        \node (cls_box) [fit=(cls), 
            rounded corners,
            fill=blue!10, inner sep=10pt] {};
        \end{scope}

        \draw[->, color=cgreen] ($(sigma_c.east)$) -- ($(cls_box.west)+(0,0.2)$);
        
        \node (y_hat)     at (8.3, 4.7) {$\hat{y}$};
        \node (R_c)     at (7.5, 4.3) {$R_c$};

        \draw[->, color=cgreen] ($(cls_box.east)+(0,0.2)$) -- (y_hat);
        \draw[->] ($(cls_box.east)+(0,-0.2)$) -- (R_c);
        \node (r_final)   at (7.5, 3) {$R_{\text{final}}=\min (R_{\sigma}, R_c)$};

        \draw[->,shorten >=2pt,shorten <=2pt] (R_sigma) |- ($(r_final.west)+(0,-0.1)$);
        \draw[->,shorten >=2pt,shorten <=2pt] (R_c.south) -- ($(r_final.north)$);

    \end{tikzpicture}
    }

    \caption{The dual RS framework. First, a RS model $g_e$ smoothed with a global $\sigma_e$ is deployed to estimate $\sigma_c(\vx)$ and return a certified radius for the estimation, $R_{\sigma}$. Second, another RS model is smoothed with $\sigma_c(\vx)$, and then perform a standard classification and return a certified radius for the classification, $R_c$. The final prediction is the result of the second stage, with a final certified radius $R_{\text{final}}=\min(R_{\sigma}, R_c)$. The green arrows indicate activated paths during inference.}
    \label{fig:ars}
    \vspace{-4mm}
\end{figure}
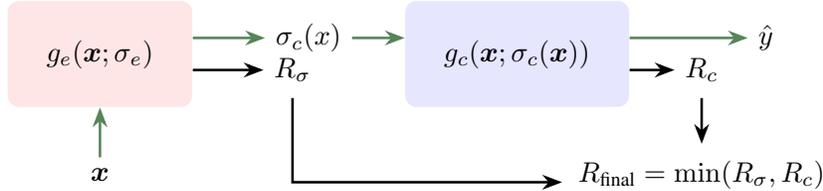

In this section, we present the dual RS framework implementing RS with input-dependent noise variances. We first give an overview of the framework, followed by details of the inference and certification process. Then, we describe the training methods to optimize the performance. Finally, we discuss an alternative routing perspective of the dual RS framework.

\subsection{Inference \& Certification}

\cref{fig:ars} illustrates the dual RS framework. Given an input $\vx$, a \emph{variance estimator} predicts an appropriate variance, $\sigma_c$, followed by a \emph{classifier} smoothed with $\sigma_c$ to perform the final classification. Intuitively, the variance estimator partitions the input space into disjoint subsets associated with different values of $\sigma_c$, and assign the input (ideally also its neighborhood) to the corresponding subset. This formulation exactly matches the definition of robustness, in the task of predicting the optimal $\sigma_c$. Therefore, a separate RS model is applied, which uses a pre-defined global noise variance to certify the estimated $\sigma_c$. Then, another base model can be smoothed via RS with $\sigma_c$ to perform certified classification. The final certified radius is then guaranteed by \cref{thm:adaptive-RS-probabilistic}. Note that we set $\sigma_e \ge \max_{\sigma_c \in \Sigma} \sigma_c$ to not limit the final certified radius inherently.

Unless otherwise noted, we use diffusion denoised smoothing to build both the variance estimator and the RS classifier, for the simplicity and efficiency. Formally, we denote the two RS models as:
\begin{align*}
    g_e(\vx, \sigma_e) := \argmax_{\sigma_i \in \Sigma} \mathbb{P}_{\bm{\delta_e} \sim \mathcal{N}(0, \sigma_e^2 I)}(h_e(\texttt{denoise}(\vx + \bm{\delta_e})) = \sigma_i), \\
    g_c(\vx, \sigma_c) := \argmax_{\hat{y} \in \gY} \mathbb{P}_{\bm{\delta_c} \sim \mathcal{N}(0, \sigma_c^2 I)}(h_c(\texttt{denoise}(\vx + \bm{\delta_c})) = \hat{y}),
\end{align*}
where \texttt{denoise} represents a single-step denoising using an off-the-shelf diffusion denoiser, and $h_e$ and $h_c$ are base models for variance estimation and classification, respectively.

At inference time, given an input $\vx$, we sample noise samples $\{\bm{\delta_e}\}$ from $\gN(0, \sigma_e^2 I)$, and use the \predict function from \citet{CohenRK19} to predict the noise variance $\sigma_c(\vx)$ with uncertainty $\alpha/2$. Then, we sample noise samples $\{\bm{\delta_c}\}$ from $\gN(0, \sigma_c(\vx)^2 I)$, and use the \predict function again to predict the class label $\hat{y}$ with uncertainty $\alpha/2$. The final prediction is $\hat{y}$, with a total uncertainty of $\alpha$, using the union bound again on the failure events, similar to the proof of \cref{thm:adaptive-RS-probabilistic}.
To certify the prediction, we use the \certify function from \citet{CohenRK19} to certify the local constancy of $\sigma_c(\vx)$ with uncertainty $\alpha/2$, and certify the classification with uncertainty $\alpha/2$. The final certified radius is $R_{\text{final}} = \min(R_{\sigma}, R_c)$, where $R_{\sigma}$ is the certified radius for the estimation of $\sigma_c(\vx)$, and $R_c$ is the certified radius for the classification. The total uncertainty is $\alpha$, as guaranteed by \cref{thm:adaptive-RS-probabilistic}. We note that for simplicity, we use the same uncertainty level $\alpha/2$ for both certifications, but they can be adjusted flexibly as long as the total uncertainty does not exceed $\alpha$.

\subsection{Training methods}

\subsubsection{Training the Variance Estimator} \label{sec:sigma_train}

\paragraph{Building the Training Dataset.} Training $h_e$ requires ground-truth labels for the optimal noise $\sigma_c(\vx)$ of each input. Given a candidate set $\Sigma$ and a fixed $h_c$, we evaluate for each input the certified radius under each $\sigma_i \in \Sigma$. The label for the optimal noise $\sigma_c(\vx)$ is then $\argmax_{i} R_c(\vx, \sigma_i)$. This step is usually the most computationally expensive part of the training, as it requires multiple certifications for each input. However, it only needs to be performed once before training $h_e$, and can be parallelized across multiple devices. In practice, we use a smaller budget $N$ than required during certification to estimate $R_c(\vx, \sigma_i)$. Specifically, a much smaller $N$ can be used to estimate $\hat{p}_A$, then this weaker estimation can be plugged into the radius formula to compute an approximation of the certified radius. In \cref{app:reduce-training-cost}, we conduct a detailed study on the effect of using a smaller budget, showing that it can significantly reduce the training cost with minimal performance degradation. As a result, we build the dataset with $N=100$ in the main experiments (c.f. \cref{sec:exp}), which matches the cost of performing a single RS inference. Another strategy to reduce the training cost is to train only on a subset of the train data, which we also study in \cref{app:reduce-training-cost}. Additionally, we discard inputs with zero certified radius across all $\sigma_i$ to reduce the noise during the training.

\paragraph{Training with Soft Labels.}
Estimating optimal variance is formulated as a classification task, but it has certain special properties. Even if the estimated $\sigma_c$ is not optimal, a non-zero certified radius is still likely. For example, assume that given $\Sigma = \{0.25, 0.5, 1.0\}$, the certified radii of $x_1$ are 0.0, 1.6 and 0.0, respectively, while those of $x_2$ are 0.3, 0.4 and 0.3, respectively. Choosing the wrong $\sigma$ for $x_1$ is more harmful than for $x_2$ intuitively, as the latter still has a reasonably close certified radius. 
Motivated by this, we propose to use soft labels introduced below to train the variance estimator. Formally, the soft label for the variance estimation is defined as:
\begin{equation*}
    y_i = \frac{\exp(R_c(\vx, \sigma_i))}{\sum_{\sigma_j \in \Sigma} \exp(R_c(\vx, \sigma_j))}.
\end{equation*}
A standard cross-entropy loss is then applied between the soft labels and the predicted class probabilities to evaluate the estimation performance. 

\paragraph{Consistency Regularization.}

Many strategies have been proposed to increase the certified radius in the standard RS training. We choose one of them, consistency regularization \citep{JeongS20}, to further improve the certified radius of the estimated $\sigma$.
Formally,
\begin{equation*}
   \mathcal{L}_{\text{con}}(\vx) := \lambda \mathbb{E}_{\delta}\left[ \mathrm{KL} (\hat{f}(\vx) \| f(\vx+\bm{\delta}))  \right] + \eta \mathrm{H}(\hat{f}(\vx)),
\end{equation*}
where $\hat{f}(x) = \mathbb{E}(f(x+\delta))$, $\mathrm{KL}$ is the Kullback-Leibler divergence, $\mathrm{H}$ is the entropy, and $\lambda$ and $\eta$ are hyperparameters controlling the trade-off between accuracy and robustness. We remark that any other RS training strategies can be alternatively applied; we choose consistency because it is the fastest while being competitive in performance (\citet{JeongKS23}, Appendix E).

\paragraph{Overall Objective.}

The overall loss function to train the variance estimator is a weighted average  between the soft-label cross-entropy loss and the consistency loss:
\begin{equation*}
    \mathcal{L}_{\sigma} = \E_{\vx} \left[ w_e(\vx)\left(\mathcal{L}_{\text{softCE}}(\vx) + w_r(\vx)\mathcal{L}_{\text{con}}(\vx)\right) \right],
\end{equation*}
where $w_e(\vx)$ and $w_r(\vx)$ are two weighting functions. We introduce a balancing weight $w_e(\vx)$ because the distribution of optimal $\sigma_c$ is usually skewed. Formally, assume the fraction of training samples with optimal noise $\sigma_i$ is $q_i$, then $w_e(\vx) = 1/q_i$ if the optimal noise for $\vx$ is $\sigma_i$. We apply two versions of the consistency regularization weight $w_r(\vx)$ in our experiments, i.e. a weaker version $w_{r,\text{weak}}(\vx)$ and a stronger version $w_{r,\text{strong}}(\vx)$. We set $w_{r,\text{weak}}(\vx) = R_c(\vx, \hat{\sigma}_{\text{min}}) / C$ and $w_{r,\text{strong}}(\vx) = R_c(\vx, \hat{\sigma}_{\text{max}}) / C$, where $\hat{\sigma}_{\text{min}}$ and $\hat{\sigma}_{\text{max}}$ are the minimum and maximum variance predicted by the variance estimator among all noisy samples respectively, and both of them are rescaled to $[0,1]$ by a constant $C$. The stronger version is more optimistic, as it asks for larger certified radius for the variance estimator, with the risk of over-regularization; the weaker version is more conservative, as it only asks for necessary certified radius for the variance estimator, with the risk of under-regularization. We find the stronger version works slightly better for easy tasks such as \cifar, while the weaker version is slightly better for harder tasks such as \IN, consistent to intuition. In many cases, both versions perform similarly (c.f. \cref{sec:delving}). We keep the choice fixed within each dataset.

\vspace{-1mm}

\subsubsection{Adapting the Classifier to the Variance Estimator} \label{sec:cls_finetune}

Prior work \citep{carlinicertified} have shown that finetuning the off-the-shelf classifier with regard to the RS framework can significantly improve the performance. In this section, we follow a similar approach, showing how to adapt the classifier to the dual RS framework.

Given a fixed variance estimator $g_e$, we finetune the classifier $h_c$ under the estimated noise variances. Formally, given an input $\vx$, we first query the noise variance $\sigma_c(\vx)$ from $g_e$. Then, we sample noise $\bm{\delta_c} \sim \gN(0, \sigma_c(\vx)^2 I)$, and apply the denoising step to obtain $\tilde{\vx} = \texttt{denoise}(\vx + \bm{\delta_c})$. Finally, we apply a standard cross-entropy loss between the prediction $h_c(\tilde{\vx})$ and the ground-truth label $y$. This procedure follows \citet{carlinicertified} with only one difference: the noise variance is input-dependent, estimated by $g_e$, instead of being a global constant.

The described training process naturally leads to an alternating training scheme, where we iteratively train the variance estimator and finetune the classifier. In practice, we find that one round of classifier finetuning is usually sufficient to achieve good performance, i.e., training the variance estimator from scratch based on the off-the-shelf classifier, followed by one round of classifier finetuning. More rounds of alternating training may lead to marginal improvements, but at a much higher computational cost (c.f. \cref{app:finetune-rounds}).

\vspace{-2mm}

\subsection{Routing with Expert RS Models} \label{sec:route}
\vspace{-1mm}

Routing is to select the best model from a pool of expert models for each input. It has been widely studied in the context of mixture-of-experts, especially for large language models \citep{varangotreille2025doinglesssurveyrouting}. In this section, we present a novel perspective of the proposed dual RS framework as a router among a pool of pretrained expert RS models.

\cref{sec:sigma_train} proposes strategies to train the variance estimator to predict the best $\sigma_c$ for a fixed base classifier $h_c$. This naturally requires $h_c$ to perform well under all $\sigma_i \in \Sigma$, each for a subset of inputs. However, as well-known in the RS literature (e.g., \citet{chenhao25acr}), no single model wins uniformly across all noise levels. Luckily, \cref{thm:adaptive-RS-probabilistic} does not restrict $h_c$ to be the same model under different $\sigma_i$. Therefore, we can define $g_c(\vx, \sigma(\vx))$ to be the best expert among a pool of models. Formally, let $\gH := \{ \gH_{\sigma_i}\}$ be the pool of the pretrained expert models where $\gH_{\sigma_i}$ are expert models performing well under $\sigma_i$. Define $\gX_{\sigma_i} := \{\vx \mid g_e(\vx, \sigma_e) = \sigma_i\}$ to be the subset of inputs assigned to $\sigma_i$ by the variance estimator. Then we define $g_c(\vx, \sigma(\vx)) := \gH_{\sigma_i}(\vx, \sigma_i)$ for all $\vx \in \gX_{\sigma_i}$. In other words, the variance estimator $g_e$ serves not only as a predictor for the optimal noise variance, but also as a router to select the best expert RS model for each input. The training process of $g_e$ remains unchanged, except that the certified radius $R_c(\vx, \sigma_i)$ is now evaluated using the corresponding expert model $\gH_{\sigma_i}$. Note that we do not evaluate the performance of the expert models except with the corresponding variance, i.e., $\gH_{\sigma_i}$ is not evaluated with $\sigma_j$ for $j \ne i$.

The proposed routing perspective of dual RS has several advantages. First, it allows leveraging existing expert models without the need for training a new base classifier that performs well under all noise levels. This is particularly useful when the training cost is prohibitively high. Second, it enables the use of specialized models that excel in specific noise regimes, potentially improving overall performance. Third, it provides a flexible framework that can easily incorporate new expert models, with the minimal effort of re-training the variance estimator. This is because certification under dual RS has much smaller overhead given the certified radii of the expert models since the variance estimator is usually lightweight. Fourth, assuming the expert models are trained independently, improving expert models usually leads to a strict improvement in the overall performance, as we will demonstrate in \cref{sec:exp}. However, due to the routing nature, the performance of dual RS is upper bounded by the performance of the expert model $\gH_{\sigma_i}$ within each $\gX_{\sigma_i}$.

As a final remark, the routing perspective of RS is not limited to the dual RS framework, and can be extended to deterministic certification methods as well. Given a pool of expert models (potentially trained with different algorithms and hyperparameters), offering different trade-offs between accuracy and robustness, one can train a standard RS model to route each input to the best expert model, then certify the routing choice by RS. The final certified radius is the minimum between the certified radius of the routing RS model and that of the selected expert model. This generalization opens up new possibilities for combining the strengths of various certification methods within a unified framework. We leave the exploration of this direction to future work.

\textbf{Comparison with Prior Works.}  \citet{mueller2021boosting} presents a similar idea which routes among a standard network and a robust network using a deterministically certified router. We generalize their idea in the following ways: (i) they only considers routing between two models due to the design of their router, while our framework allows routing among multiple models natively; (ii) they uses a deterministic certification method to certify a heuristically trained router, while our framework uses RS to train and certify the router, which is more scalable and flexible; (iii) they focus on improving the accuracy-robustness trade-off under the given radii, while our objective is to optimize the overall performance across all radii.

\vspace{-2mm}

\section{Experimental Evaluation}\label{sec:exp}
\vspace{-1mm}

In this section, we extensively evaluate the proposed dual RS method on \cifar and \IN. The results demonstrate that dual RS can achieve strong performance across different radii, which is unattainable with a global noise variance. Further, it incurs only a modest computational overhead compared to standard RS. We include all implementation details in \cref{app-exp} and only highlight key experimental settings here.

\textbf{Baselines.}
We compare our method against two baselines: (i) diffusion denoised smoothing with global noise variances \citep{carlinicertified}, which we use as the base classifiers, and (ii) the state-of-the-art input-dependent RS method \citep{jeong2024multi}, denoted as \emph{Multiscale}. Unless otherwise stated, all results are reported with $N=10{,}000$ noise samples for certification with the overall uncertainty level $\alpha=0.001$.

\textbf{\cifar Setup.} Unless otherwise stated, $\Sigma$ is set to $\{0.25, 0.5, 1.0\}$. Following the baselines, we employ a 50M-parameter diffusion model \citep{nichol2021improved} as the denoiser, and a 87M-parameter ViT model \citep{dosovitskiy2020image} as the classifier. A ResNet-110 \citep{he2016deep} is used as the base model for the variance estimator, and $N = 100$ is used to estimate $R_c(\vx; \sigma_i)$ during training.

\textbf{\IN Setup.} Unless otherwise stated, $\Sigma$ is set to $\{0.5, 1.0\}$. Following \citet{carlinicertified}, we utilize a 552M-parameter class-unconditional diffusion model \citep{dhariwal2021diffusion} as the denoiser and a 305M-parameter BEiT model \citep{bao2021beit} as the classifier. A ResNet-50 is used as the variance estimator, and $N = 100$ is used to estimate $R_c(\vx; \sigma_i)$ during training.

\subsection{Key Results}

\begin{figure*}[]
    \centering
    \begin{subfigure}[t]{.37\linewidth}
        \centering
        \includegraphics[width=.92\linewidth]{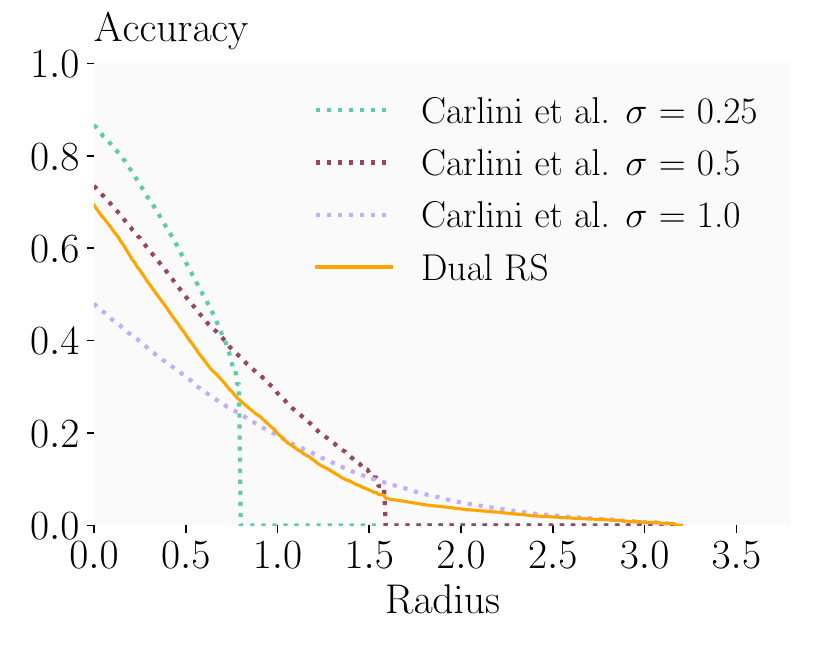}
        \vspace{-3mm}
        \caption{Comparison with Standard RS} \label{subfig:against_pretrained_base_model}
        \vspace{-1mm}
    \end{subfigure}
    \begin{subfigure}[t]{.37\linewidth}
        \centering
        \includegraphics[width=.92\linewidth]{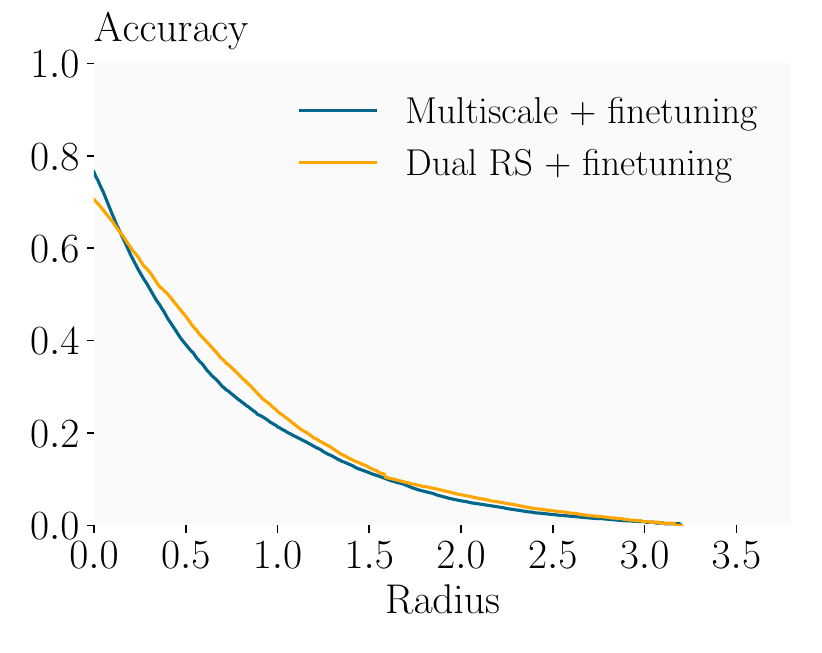}
        \vspace{-3mm}
        \caption{Comparison with Input-dependent RS}
        \label{subfig:against_multiscale}
        \vspace{-1mm}
    \end{subfigure}
    \caption{Certified accuracy on \cifar across radii.}
    \label{fig:main_result}
    \vspace{-2mm}
\end{figure*}

\begin{table}[]
    \centering
    \caption{Certified accuracy on \cifar across different certification radii. \textbf{Bold} entries indicate whenever Dual RS outperforms Multiscale.}
    \label{tab:main_result}
    \vspace{-2mm}
    \resizebox{0.95\linewidth}{!}{
    \begin{tabular}{cccccccccccccc}
        \toprule
         \multicolumn{2}{c}{Method}                          & $\sigma$                     & 0.00 & 0.25 & 0.50 & 0.75 & 1.00 & 1.25 & 1.50 & 1.75 & 2.00 & 2.25 & 2.50 \\ \midrule
         \multirow{4}{*}{Off-the-shelf} & \multirow{3}{*}{Carlini et al.} & 0.25                         & 86.61 & 73.90          & 57.02          & 35.30 & 0.00 & 0.00 & 0.00 & 0.00 & 0.00 & 0.00 & 0.00 \\ %
                                        &  & 0.5                            & 73.49 & 62.23          & 49.46          & 38.20 & 28.58 & 19.54 & 11.22 & 0.00 & 0.00 & 0.00 & 0.00 \\ %
                                         & & 1.0                            & 47.98 & 39.85          & 32.17          & 25.16 & 19.34 & 14.49 & 10.30 & 7.32 & 4.89 & 3.35 & 2.15 \\ 
          & Ours     & \{0.25, 0.5, 1.0\}                   & 68.34 & 55.25 & 41.28 & 29.01 & 19.85 & 12.73 & 7.62 & 4.73 & 3.54 & 2.62 & 1.83 \\ \cmidrule{1-14}
          \multirow{2}{*}{Adaptive finetuning} & Multiscale                   & \{0.25, 0.5, 1.0\}   & 76.51 & 54.78          & 39.15          & 28.46 & 21.33 & 15.95 & 11.40 & 7.91 & 5.31 & 3.63 & 2.34 \\  %
          & Ours    &   \{0.25, 0.5, 1.0\}    & 70.53 & \textbf{57.48} & \textbf{45.27} & \textbf{34.15} & \textbf{24.68} & \textbf{17.84} & \textbf{12.46} & \textbf{8.83} & \textbf{6.65} & \textbf{4.73} & \textbf{3.14} \\ \bottomrule
    \end{tabular}
    }
    \vspace{-5mm}
\end{table}

\paragraph{Dual RS with Single Pretrained Classifier.}
We first evaluate the performance of dual RS with a pretrained global classifier, as described in \cref{sec:sigma_train}. 
\cref{subfig:against_pretrained_base_model} and \cref{tab:main_result} compare the pretrained diffusion denoised smoothing model with different global noise variances and dual RS using the same classifier on \cifar. While the baseline models with a small global noise variance (e.g., $\sigma=0.25$ or $0.5$) achieve high certified accuracy at small radii, they fail to provide non-trivial guarantees at larger radii. Conversely, the model with a large global noise variance ($\sigma=1.0$) attains good certified accuracy at large radii but suffers from low accuracy at small radii. In contrast, dual RS has a strong performance across all radii, achieving a superior accuracy-robustness trade-off. This demonstrates that dual RS can effectively leverage the pretrained classifier.

\vspace{-3mm}

\paragraph{Dual RS with Single Classifier Finetuning.}
\cref{subfig:against_multiscale} and \cref{tab:main_result} compare dual RS with Multiscale \citep{jeong2024multi}, the state-of-the-art input-dependent RS method. For a fair comparison, we finetune the classifier in dual RS as described in \cref{sec:cls_finetune}, while Multiscale adopts the finetuned diffusion denoiser described in \citet{jeong2024multi}. The result shows that dual RS consistently outperforms Multiscale across most radii, with especially strong improvements in the small-radius region. At radii $0.5$, $0.75$, and $1.0$, dual RS improves certified accuracy by $15.6\%$, $20.0\%$, and $15.7\%$, respectively.
On a single NVIDIA RTX 4090 GPU with batch size 1000 and $N=10{,}000$, certifying with dual RS requires 22.58 seconds per input on average, compared to 14.07 seconds for standard RS and 20.21 seconds for Multiscale. Thus, dual RS incurs only a modest computational overhead relative to standard RS, while achieving significant performance gains. We remark that Multiscale requires multiple rounds of certification for inputs with small certified radii, leading to a higher worst-case certification time (14.07, 28.14, and 42.21 seconds on average for 1, 2, and 3 rounds, respectively), while dual RS spends a fixed amount of time for all inputs.

\begin{figure*}[]
    \centering
    \includegraphics[width=.34\linewidth]{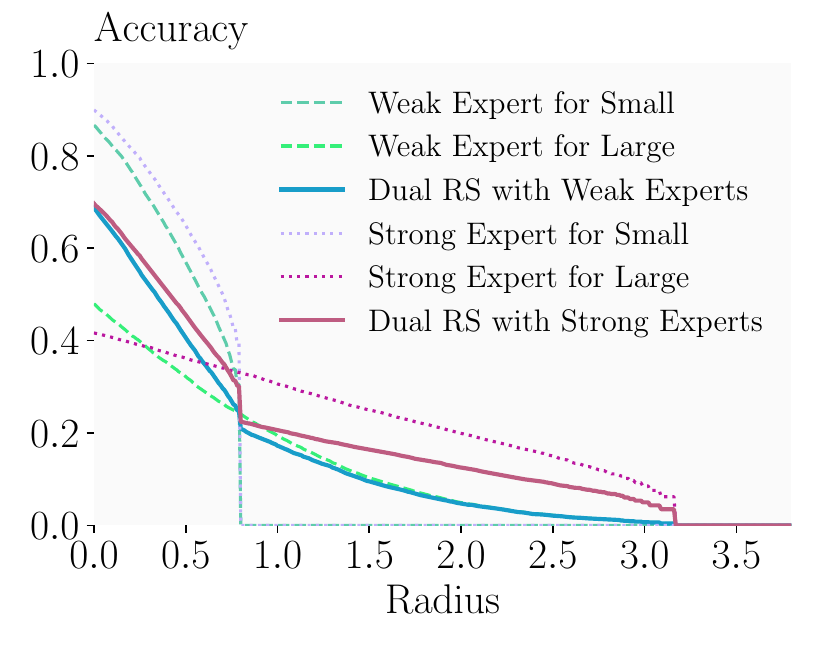}
    \vspace{-4mm}
    \caption{Comparison between dual RS built on weak and strong experts, respectively, along with the experts.}
    \label{fig:comp_route}
    \vspace{-3mm}
\end{figure*}

\begin{table}[]
    \centering
    \caption{Certified accuracy on \IN, structured in the same way as \cref{tab:main_result}. }
    \label{tab:main_result_in}
    \vspace{-2mm}
    \resizebox{0.65\linewidth}{!}{
    \begin{tabular}{cccccccc}
        \toprule
       \multicolumn{2}{c}{Method}                            & $\sigma$           & 0.0  & 0.5  & 1.0  & 1.5  & 2.0  \\ \midrule
       \multirow{3}{*}{Off-the-shelf} & \multirow{2}{*}{Carlini at al.}   & 0.5                & 74.4 & 64.4 & 52.4 & 34.8 & 0.0  \\ 
                                        & & 1.0                & 56.0 & 47.8 & 37.4 & 29.4 & 24.0 \\ %
        & Ours & \{0.5, 1.0\}       & 67.8 & 54.6 & 40.4 & 28.0 & 15.6 \\ \cmidrule{1-8}
        \multirow{2}{*}{Adaptive finetuning} & Multiscale       & \multirow{2}{*}{\{0.5, 1.0\}} & 75.0 & 55.8 & 41.0 & 30.8 & 25.2 \\ %
        & Ours   &        & 74.0 & \textbf{60.6} & \textbf{48.0} & \textbf{33.6} & 17.0 \\ \bottomrule
    \end{tabular}
    }
    \vspace{-7mm}
\end{table}

\textbf{Dual RS with Multiple Pretrained Experts (Routing).} 
We further evaluate the efficacy of dual RS as a routing mechanism with multiple pretrained expert classifiers, as discussed in \cref{sec:route}. Specifically, we consider two experts: one specialized for $\sigma=0.25$ and another specialized for $\sigma=1.0$. For $\sigma=0.25$, we define the \emph{Weak Expert for Small} to be an off-the-shelf denoised smoothing model, and the \emph{Strong Expert for Small} to be a finetuned denoised smoothing model by \citet{carlinicertified} on $\sigma=0.25$. For $\sigma=1.0$, we define the \emph{Weak Expert for Large} to be the same off-the-shelf model, and the \emph{Strong Expert for Large} to be another off-the-shelf model, pretrained by \citet{chenhao25acr}, which achieves the state-of-the-art performance for large radii. We found that incorporating the weight $w_e(\vx)$ degrades performance on strong experts, so we set $w_e(\vx)=1$ for all $\vx$ in this experiment. \cref{fig:comp_route} compares these four experts and dual RS built upon weak and strong experts, respectively. The results show that dual RS effectively leverages the improved performance of the strong experts, achieving a better accuracy-robustness trade-off than that of the weak experts. This demonstrates that dual RS can flexibly incorporate different expert models to further enhance performance.

\textbf{Dual RS on Large Datasets.} \label{sec:in_exp}
We further evaluate dual RS on \IN. As shown in \cref{tab:main_result_in}, Dual RS achieves strong certified accuracy in the medium and small radii region. Compared to Multiscale, the state-of-the-art input-dependent RS, dual RS gets 8.6\%, 17.1\% and 9.1\% performance advantage at radii 0.5, 1.0 and 1.5 respectively. Overall, these results show that Dual RS scales effectively to large datasets and high-dimensional input spaces.

We further conduct ablation studies on three aspects: (1) different choices of $\sigma$ candidate sets, (2) strategies for constructing the train set of the variance estimator, and (3) different architectures of the variance estimator, detailed in \cref{app:sigma-candidates}, \cref{app:reduce-training-cost}, and \cref{app:smaller-estimator}, respectively. The key findings are: (1) the candidate set $\Sigma$ strongly affects the favored radii, similar to the observation in the global variance methods; (2) the variance estimator can be trained with minimal performance degradation on a much smaller $N$ than certification (up to 99\% cost reduction) or a much smaller train set (up to 80\% cost reduction); and (3) the performance of dual RS is robust to the architecture of the variance estimator.
\vspace{-3mm}
\subsection{Delving into Dual RS} \label{sec:delving}

\begin{figure*}[]
    \centering
    \begin{subfigure}[t]{.3\linewidth}
        \centering
        \includegraphics[width=.9\linewidth]{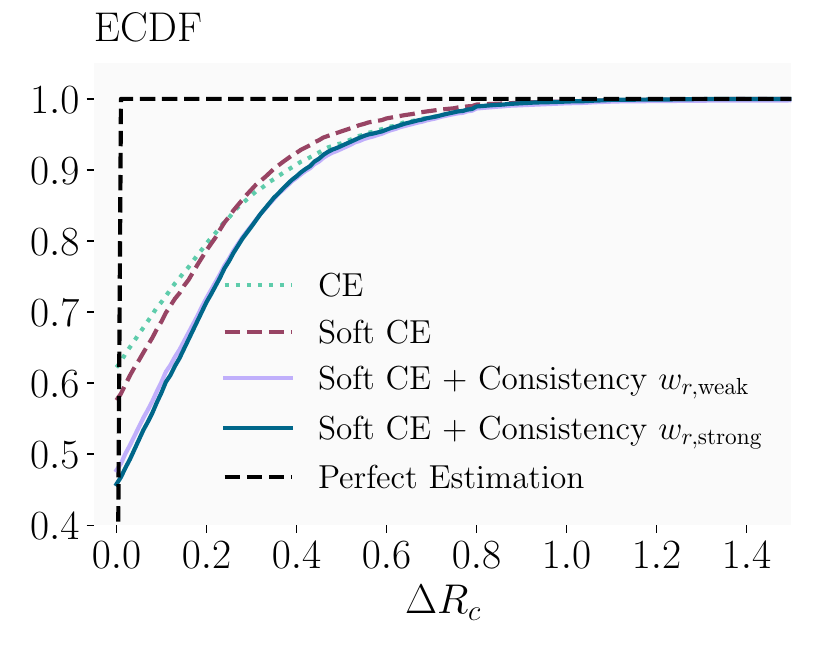}
        \vspace{-2mm}
        \caption{}
        \label{subfig:rloss}
    \end{subfigure}
    \begin{subfigure}[t]{.3\linewidth}
        \centering
        \includegraphics[width=.9\linewidth]{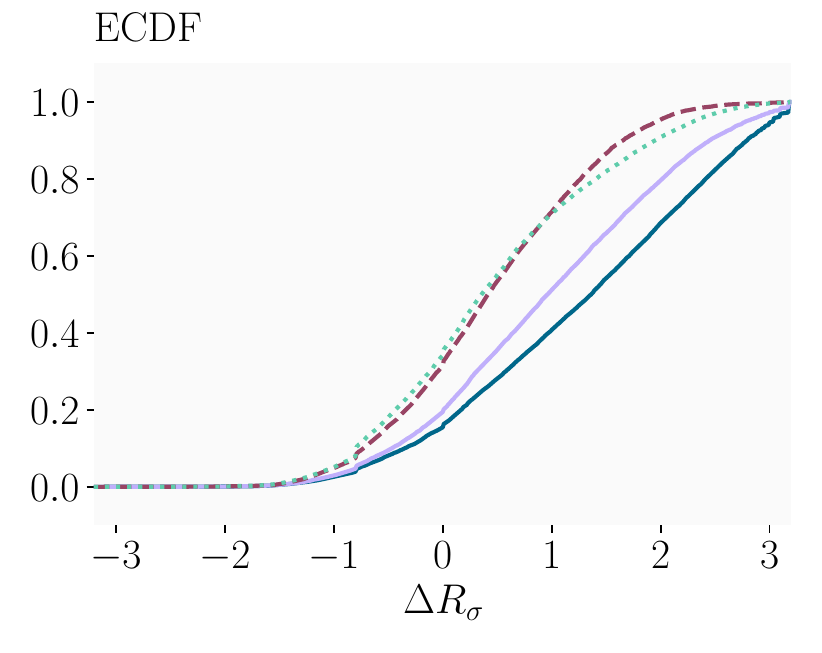}
        \vspace{-2mm}
        \caption{}
        \label{subfig:align}
    \end{subfigure}
    \begin{subfigure}[t]{.3\linewidth}
        \centering
        \includegraphics[width=.9\linewidth]{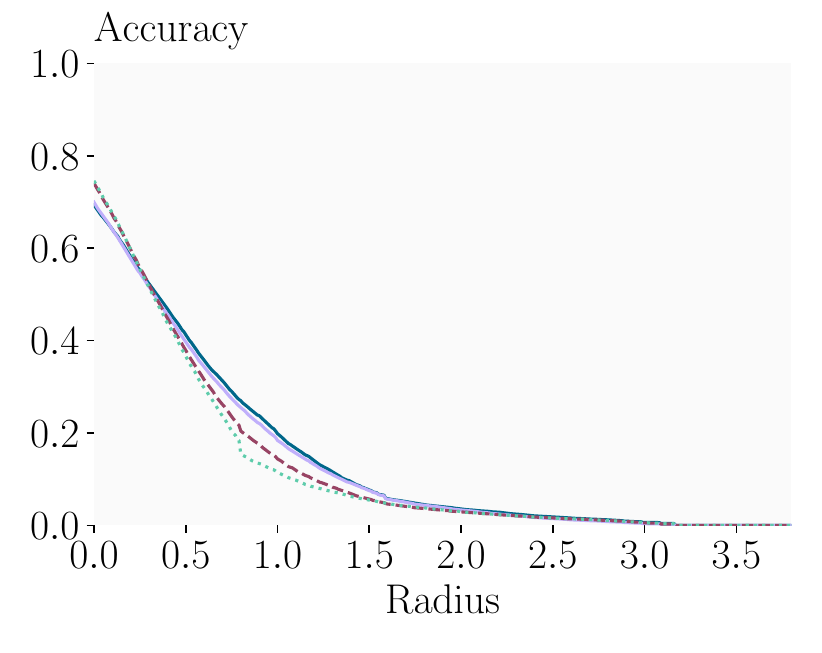}
        \vspace{-2mm}
        \caption{}
        \label{subfig:abla_loss}
    \end{subfigure}
    \vspace{-3mm}
    \caption{Comparison of dual RS models with different variance estimators.}
    \label{fig:abla_2radius}
    \vspace{-6.5mm}
\end{figure*}

In this section, we conduct an in-depth study on dual RS on \cifar. We use diffusion denoised smoothing as the classifier with off-the-shelf models and $\Sigma = \{0.25, 0.5, 1.0\}$.

To evaluate the performance of the variance estimation, we define $\Delta{R_c} := R_{c}^*(x)-R_c(x)$, where $R_{c}^*(x)$ is the maximum classification certified radius among all candidate noise variances for input $x$. This metric reflects how much $R_c$ is reduced due to the suboptimal variance estimation. \cref{subfig:rloss} shows the empirical cumulative distribution function (ECDF) of this metric for variance estimators trained with different loss functions and hyperparameters on \cifar. The intercept at $\Delta{R_c}=0$ indicates the proportion of samples for which the variance estimator predicts the optimal noise variance, and the area between the curve and the the perfect estimation (black dash line) reflects the overall loss in the certified radius due to suboptimal variance estimation. We observe that using soft cross-entropy (CE) loss instead of standard CE loss reduces the variance estimation accuracy, as it encourages the model to predict a suboptimal noise variance that yields a similar certified radius rather than the optimal one. However, fewer inputs are constrained significantly when using soft CE loss, as the curve is closer to the perfect estimation line when $\Delta R_c$ is large. Further, adding the consistency loss reduces the variance estimation accuracy in general, since it puts additional regularization on the robustness of the variance estimator.

Since the final certified radius is the minimum between the classification certified radius and the variance certified radius, the alignment between these two radii is of interest as well. We define $\Delta{R_{\sigma}} $ as $\Delta{R_{\sigma}} = R_{\sigma} - R_c$. A negative $\Delta{R_{\sigma}}$ means that the final radius is constrained by the $R_{\sigma}$, while a positive $\Delta{R_{\sigma}}$ means it is constrained by $R_c$. Ideally, we want $\Delta{R_{\sigma}}$ to be positive for as many samples as possible, so that the final certified radius is not constrained by $R_{\sigma}$. \cref{subfig:align} shows the ECDF of $\Delta{R_{\sigma}}$. The intercept at $\Delta{R_{\sigma}}=0$ indicates the proportion of samples constrained by $R_{\sigma}$. We observe that using soft CE loss decreases this ratio, and adding consistency loss further decreases it significantly. Furthermore, using the stronger version of the consistency weight, less samples are constrained by $R_{\sigma}$. This aligns with our intuition in the design.

As a reference, \cref{subfig:abla_loss} shows the accuracy-radius curves for these models. Using soft CE loss almost improves over the standard CE loss uniformly, while adding consistency loss slightly degrades the performance at small radii but improves it at large radii. Overall, the model trained with soft CE loss and consistency loss achieves the best accuracy-robustness trade-off. 

\vspace{-2mm}

\section{Conclusion}
\vspace{-2mm}
In this work, we address the fundamental trade-off between certified accuracy and certified radius in Randomized Smoothing (RS). We prove that RS remains valid under input-dependent noise variances, provided the variance is locally constant within the certified region. Building on this result, we introduce a dual RS framework, which achieves strong performance across both small and large radii, unattainable with fixed noise variance, while incurring modest computational overhead. Our method consistently outperforms prior input-dependent noise approaches across most radii. Further, the dual RS framework offers a novel routing perspective for certified robustness, enhancing the accuracy-robustness trade-off using off-the-shelf expert RS models.

\section{Reproducibility Statement}

We have made extensive efforts to ensure the reproducibility of our work. For theoretical results, we provide precise definitions and formal statements in \cref{sec:theory}, with complete proofs given in \cref{app:proofs}. For the proposed framework, detailed descriptions of the inference, certification, and training procedures are presented in \cref{sec:ars_frame}. Experimental settings, including architectures, hyperparameters, and training details, are reported in \cref{sec:exp} and \cref{app-exp}. To further support reproducibility, we include the link to our code and data in the abstract.

\section*{Acknowledgements}

This work has been done as part of the EU grant ELSA (European Lighthouse on Secure and Safe AI, grant agreement no. 101070617) and the SERI grant SAFEAI (Certified Safe, Fair and Robust Artificial Intelligence, contract no. MB22.00088). Views and opinions expressed are however those of the authors only and do not necessarily reflect those of the European Union or European Commission. Neither the European Union nor the European Commission can be held responsible for them. 

The work has received funding from the Swiss State Secretariat for Education, Research and Innovation (SERI).

\section*{Proof Formalization}
All theoretical results and proofs are formalized in Lean using Paper Formalizer~\citep{mao2026paperformalizer}, and are available at \url{https://github.com/AlgebraLoveme/dual-rs-formalization}.

\bibliography{iclr2026_conference}
\bibliographystyle{iclr2026_conference}

\newpage
\appendix
\section{Usage of Large Language Models}

We used a large language model (GPT-5) solely to assist with polishing and grammar correction of the paper. The LLM was not involved in other aspects of this paper.

\begin{table}[]
    \centering
    \caption{GPU hours on a single NVIDIA RTX 4090 for the main components of our training pipeline. The costs of building the optimal-variance dataset, training variances estimator and finetuning the classifier are reported as (number of parallel GPUs $\times$ wall-clock time in hours).}
    \label{tab:training_cost}
    \resizebox{.65\linewidth}{!}{
    \begin{tabular}{ccc}
        \toprule
         Component   & \cifar & \IN  \\ \midrule
        Building optimal variances dataset & 1 $\times$ 6.0  &  128 $\times$ 42.2  \\ 
        Training variance estimator        & 1 $\times$ 2.5  &  8 $\times$ 63   \\ 
        Generating estimated variances     & 1.5  &  9.9  \\ 
        Finetuning the classifier         & 1 $\times$ 1.0   &  8 $\times$ 0.7\\ \bottomrule
    \end{tabular}
    }
\end{table}

\section{Deferred Proofs} \label{app:proofs}

\subsection{Proof of \cref{thm:adaptive-RS-deterministic}}\label{app:proof-of-adaptive-RS-deterministic}

We first cite the following lemma from \citet{SalmanLRZZBY19}, using the formulation as in Lemma D.1 of \citet{jeong2024multi}. Note that we adapt the notation to be consistent with our paper. We slightly abuse the notation and let $p_\sigma$ be the probability of a certain class, i.e., $p_\sigma(\vx):= \Prob_{\bm{\delta} \sim \gN(\bm{0}, \sigma^2 \bm{I})}(f_c(\vx + \bm{\delta}) = y)$ for some $y$.

\begin{restatable}{lem}{RSlemma}\label{lem:RS}
    $h_y(\vx):= \sigma \Phi^{-1}(p_\sigma(\vx))$ is $1$-Lipschitz with respect to the $\ell_2$ norm.
\end{restatable}

\cref{lem:RS} can be extended to locally constant $\sigma(\vx)$ as follows.

\begin{restatable}{lem}{adaptiveRSlemma}\label{lem:adaptive-RS}
    Let $\gX$ be partitioned into disjoint subsets $\gX = \bigcup_{i \in I} \gX_i$, and assume $\sigma(\vx)$ is constant within each $\gX_i$. Let $h_y(\vx):= \sigma(\vx) \Phi^{-1}(p_{\sigma(\vx)}(\vx))$. Then for all $i \in I$, $\forall \vx_1, \vx_2 \in \gX_i$, we have $|h_y(\vx_1) - h_y(\vx_2)| \le \|\vx_1 - \vx_2\|_2$.
\end{restatable}

\begin{proof}
    For any $i \in I$, let $\sigma_i$ be the constant value of $\sigma(\vx)$ for $\vx \in \gX_i$. Then for any $\vx_1, \vx_2 \in \gX_i$, we have
    \begin{align*}
        |h_y(\vx_1) - h_y(\vx_2)| &= |\sigma_i \Phi^{-1}(p_{\sigma_i}(\vx_1)) - \sigma_i \Phi^{-1}(p_{\sigma_i}(\vx_2))| \\
        &\le \|\vx_1 - \vx_2\|_2,
    \end{align*}
    where the last inequality follows from \cref{lem:RS}.
\end{proof}

Now we are ready to prove \cref{thm:adaptive-RS-deterministic}, restated below for convenience.

\adaptiveRSdeterministic*

\begin{proof}
    Let $\gX_i = \{ \vx | \sigma(\vx) = \sigma_i \}$, where $\sigma_i$ are distinct values taken by $\sigma(\vx)$. Then $\gX = \bigcup_{i \in I} \gX_i$ is a partition of $\gX$ into non-overlapping subsets. By \cref{lem:adaptive-RS}, for any $i \in I$, $\forall \vx_1, \vx_2 \in \gX_i$, we have $|h_y(\vx_1) - h_y(\vx_2)| \le \|\vx_1 - \vx_2\|_2$. Further, given $\vx_0$, there exists exactly one $j \in I$ such that $\vx_0 \in \gX_j$. The local constancy assumption implies $\sB(\vx_0, R_\sigma) \subseteq \gX_j$. If there is no adversarial perturbation $\bm{\delta}$ such that $\|\bm{\delta}\|_2 < R_\sigma$ and $g_c(\vx_0 + \bm{\delta}) \ne g_c(\vx_0)$, then the claim holds trivially. In the following, we consider the case where such adversarial perturbation $\bm{\delta}$ exists.

    Given the smoothing distribution $\gN(\bm{0}, \sigma^2 \bm{I})$ where $\sigma = \sigma_j$, let $A$ be the most likely class at $\vx_0$, and $B$ be any other class. Let $p_A(\vx)$ be the probability of class $A$ at $\vx$ under the smoothing distribution, and $p_B(\vx)$ be the probability of class $B$ at $\vx$. 
    Therefore, $\forall \vx_0 + \bm{\delta} \in \gX_j$, we have $\sigma \Phi^{-1}(p_A(\vx_0)) - \sigma \Phi^{-1}(p_A(\vx_0 + \bm{\delta})) = h_A(\vx_0) - h_A(\vx_0 + \bm{\delta}) \le \|\bm{\delta}\|_2$. Let $\bm{\delta}$ be an adversarial perturbation such that $\|\bm{\delta}\|_2 < R_\sigma$ and let $B$ be the most likely class at $\vx_0 + \bm{\delta}$. Then, since $p_A(\vx_0 + \bm{\delta}) \le p_B(\vx_0 + \bm{\delta})$ and $\Phi^{-1}(t)$ is monotonically increasing in $t$, we have 
    \begin{align*}
        \sigma \Phi^{-1}(p_A(\vx_0)) - \sigma \Phi^{-1}(p_B(\vx_0 + \bm{\delta})) &\le \sigma \Phi^{-1}(p_A(\vx_0)) - \sigma \Phi^{-1}(p_A(\vx_0 + \bm{\delta})) \\
        &\le \|\bm{\delta}\|_2.
    \end{align*}
    Further, applying \cref{lem:adaptive-RS} again gives $\sigma \Phi^{-1}(p_B(\vx_0 + \bm{\delta})) - \sigma \Phi^{-1}(p_B(\vx_0)) = h_B(\vx_0 + \bm{\delta}) - h_B(\vx_0) \le \|\bm{\delta}\|_2$. Combining the two inequalities gives
    $$
    \sigma \Phi^{-1}(p_A(\vx_0)) - \sigma \Phi^{-1}(p_B(\vx_0)) \le 2\|\bm{\delta}\|_2.
    $$
    Thus, we have
    \begin{align*}
        \|\bm{\delta}\|_2 &\ge \frac{\sigma}{2} \left( \Phi^{-1}(p_A(\vx_0)) - \Phi^{-1}(p_B(\vx_0)) \right) \\
        &\ge \frac{\sigma}{2} \left( \Phi^{-1}(p_A(\vx_0)) - \Phi^{-1}(1 - p_A(\vx_0)) \right) \\
        &= \sigma \Phi^{-1}(p_A(\vx_0)) \\
        &= R(\vx_0, \sigma_j) \\
        &= R(\vx_0, \sigma(\vx_0)).
    \end{align*}
    This completes the proof.
\end{proof}

\subsection{Proof of \cref{thm:adaptive-RS-probabilistic}}\label{app:proof-of-adaptive-RS-probabilistic}

We restate \cref{thm:adaptive-RS-probabilistic} below for convenience.

\adaptiveRSprobabilistic*

\begin{proof}
    Let $E_c$ be the event that the fixed-noise smoothed classifier $g_c(\cdot, \sigma(\vx_0))$ is robust throughout $\sB(\vx_0, R_c)$, i.e.,
    $$
        g_c(\vx, \sigma(\vx_0)) = g_c(\vx_0, \sigma(\vx_0))
        \quad \text{for all } \vx \in \sB(\vx_0, R_c).
    $$
    Let $E_\sigma$ be the event that $\sigma(\vx)$ is constant throughout $\sB(\vx_0, R_\sigma)$, i.e.,
    $$
        \sigma(\vx) = \sigma(\vx_0)
        \quad \text{for all } \vx \in \sB(\vx_0, R_\sigma).
    $$
    By assumption, $\Prob(E_c) \ge 1 - \alpha$ and $\Prob(E_\sigma) \ge 1 - \beta$. Therefore, by the union bound,
    $$
        \Prob(E_c \cap E_\sigma)
        = 1 - \Prob(E_c^c \cup E_\sigma^c)
        \ge 1 - \Prob(E_c^c) - \Prob(E_\sigma^c)
        \ge 1 - \alpha - \beta.
    $$
    On the event $E_c \cap E_\sigma$, for any $\vx$ such that $\|\vx - \vx_0\|_2 < \min(R_c, R_\sigma)$, we have $\vx \in \sB(\vx_0, R_c)$ and $\vx \in \sB(\vx_0, R_\sigma)$. Hence,
    $$
        g_c(\vx, \sigma(\vx))
        = g_c(\vx, \sigma(\vx_0))
        = g_c(\vx_0, \sigma(\vx_0)),
    $$
    where the first equality follows from $E_\sigma$ and the second from $E_c$. This proves the claim with probability at least $1 - \alpha - \beta$.
\end{proof}

\section{Experiment Details}\label{app-exp}

\subsection{Experiment Setup}
\paragraph{\cifar} In the main experiments, the variance estimator model is trained from scratch for 90 epochs with a batch size of 256. We use the AdamW optimizer with an initial learning rate of 0.01 and a weight decay of 0.01. The learning rate is decayed by a factor of 0.5 every 30 epochs. Unless otherwise stated, we set $\lambda = 40$ and $\eta = 0.5$ for the consistency loss, and use $\sigma_e = 1.0$ for variance estimation certification. To compute the consistency loss, we always use two noise samples ($m=2$) following \citet{JeongS20}. For classifier finetuning, we apply the Cross-Entropy loss on denoised images $\texttt{denoise}(\vx + \bm{\delta})$. The classifier is finetuned for 15 epochs with a batch size of 128 using AdamW with a learning rate of $2\times 10^{-5}$ and a weight decay of 0.01.
\paragraph{\IN} The variance estimator model is trained from scratch for 9 epochs with a batch size of 200. We use the AdamW optimizer with an initial learning rate of 0.005 and a weight decay of 0.01. The learning rate is decayed by a factor of 0.5 every 3 epochs. Unless otherwise stated, we set $\lambda = 10$ and $\eta = 0.5$ for the consistency loss, and use $\sigma_e = 1.0$ for variance estimation certification. For the classifier finetuning, we randomly choose 2\% of the training set to finetune the classifier for 1 epoch with a batch size of 32 using AdamW with a learning rate of $2\times 10^{-5}$ and a weight decay of 0.01. On \IN, after finetuning the classifier, we do not retrian the variance estimator due to the high computational cost.

\subsection{Training Cost}
We report the computational cost of training our dual RS framework. All experiments are conducted on NVIDIA RTX 4090 GPUs.Table~\ref{tab:training_cost} summarizes the GPU hours required for each major component of the training pipeline. The highest cost arises from constructing the optimal variances dataset, which involves performing classification certification on the full dataset under all three noise variances. For the \cifar main experiments, two variance estimators are trained and one classifier finetuning is performed, resulting in approximately 19.5 GPU hours on a single RTX 4090. 

In practice on \IN, we parallelized the dataset construction step across 128 GPUs. The overall training pipeline requires approximately 115.8 hours.

\section{Numerical Examples for the Confidence Penalty}\label{app:deferred-tables-figures}

We list numerical examples of the confidence penalty $\beta$ under different uncertainty levels in \cref{tab:beta_cost}.

\section{Additional Studies} \label{app:additional-studies}

In this section, we present additional studies to further investigate different components of our dual RS framework on \cifar. 

\subsection{Ablation on Consistency Loss Hyperparameter $\lambda$}

\begin{table}
    \centering
    \caption{Numerical examples of the confidence penalty $\beta$ under the given uncertainty level. We assume large enough $R_\sigma$, such that the final certified radius equals $R_c$. When $\alpha:\beta=1:0$, it matches the standard RS setting. The budget is fixed to $N=10^5$, $\sigma$ is fixed to 1.0 and $\alpha+\beta$ is fixed to $0.001$, following the standard certification setting.} \label{tab:beta_cost}
    \resizebox{.3\linewidth}{!}{
    \begin{tabular}{ccc}
    \toprule
    $\alpha:\beta$ & $\hat{p_\sigma}$ & certified radius \\ \midrule
    $1:0$       & $0.99$   &  2.2900             \\
    $1:1$       & $0.99$   &  2.2877             \\ 
    $1:4$       & $0.99$   &  2.2848             \\ 
    \cmidrule{1-3}
    $1:0$       & $0.8$   &    0.8277           \\
    $1:1$       & $0.8$   &    0.8267          \\
    $1:4$       & $0.8$   &    0.8256           \\ 
    \cmidrule{1-3}
    $1:0$       & $0.6$   &    0.2409           \\
    $1:1$       & $0.6$   &    0.2401           \\
    $1:4$       & $0.6$   &    0.2391           \\
    \bottomrule
    \end{tabular}
    }
\end{table}

\begin{figure*}[]
    \centering
    \begin{subfigure}[t]{.32\linewidth}
        \centering
        \includegraphics[width=.95\linewidth]{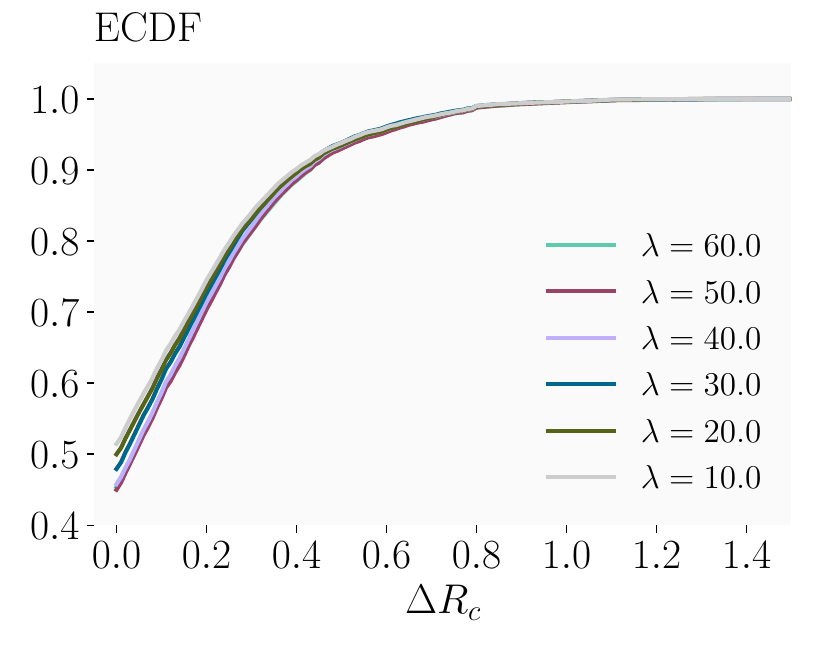}
        \caption{}
        \label{subfig:abl_rloss}
    \end{subfigure}
    \begin{subfigure}[t]{.32\linewidth}
        \centering
        \includegraphics[width=.95\linewidth]{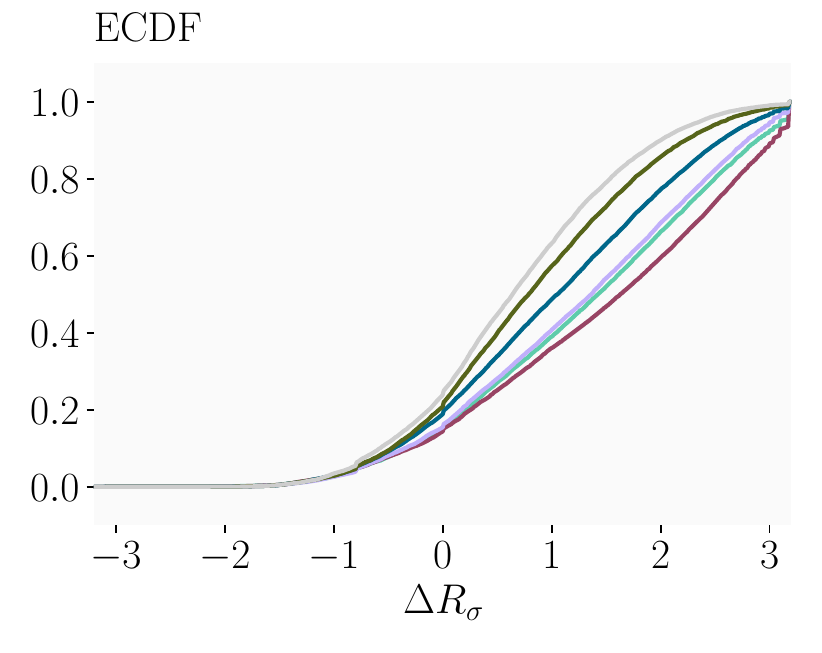}
        \caption{}
        \label{subfig:abl_align}
    \end{subfigure}
    \begin{subfigure}[t]{.32\linewidth}
        \centering
        \includegraphics[width=.95\linewidth]{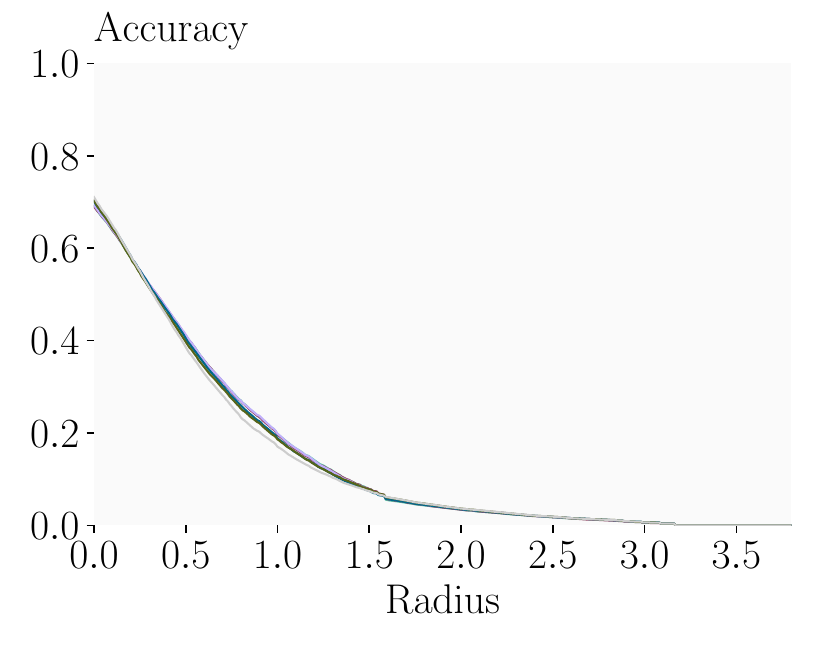}
        \caption{}
        \label{subfig:abl_abla_loss}
    \end{subfigure}
    \caption{Abaltion study on $\lambda$ in the consistency loss. The figures are organized in the same way as \cref{fig:abla_2radius}.}
    \label{fig:abla_lbd}
    \vspace{-4mm}
\end{figure*}

\begin{figure*}[]
    \centering
    \includegraphics[width=.45\linewidth]{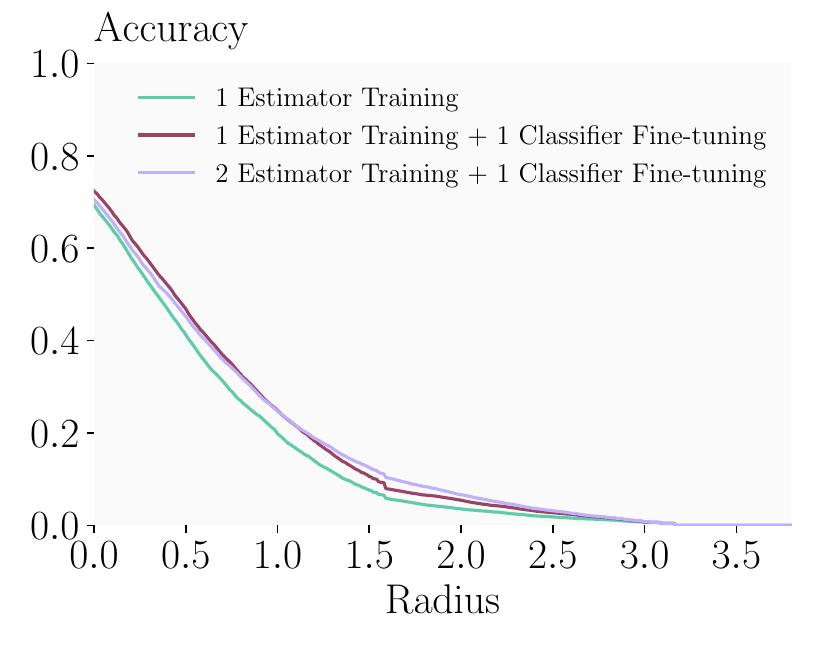}
    \caption{Effect of training rounds for the variance estimator and the classifier. \emph{1 Estimator Training} trains the variance estimator using the off-the-shelf classifier. \emph{1 Estimator Training + 1 Classifier Finetuning} finetunes the classifier using the estimated variances by the trained variance estimator. \emph{2 Estimator Training + 1 Classifier Finetuning} further re-train the variance estimator based on the finetuned classifier.}
    \label{fig:abla_ft}
    \vspace{-2mm}
\end{figure*}

\begin{figure*}[]
    \centering
    \includegraphics[width=.45\linewidth]{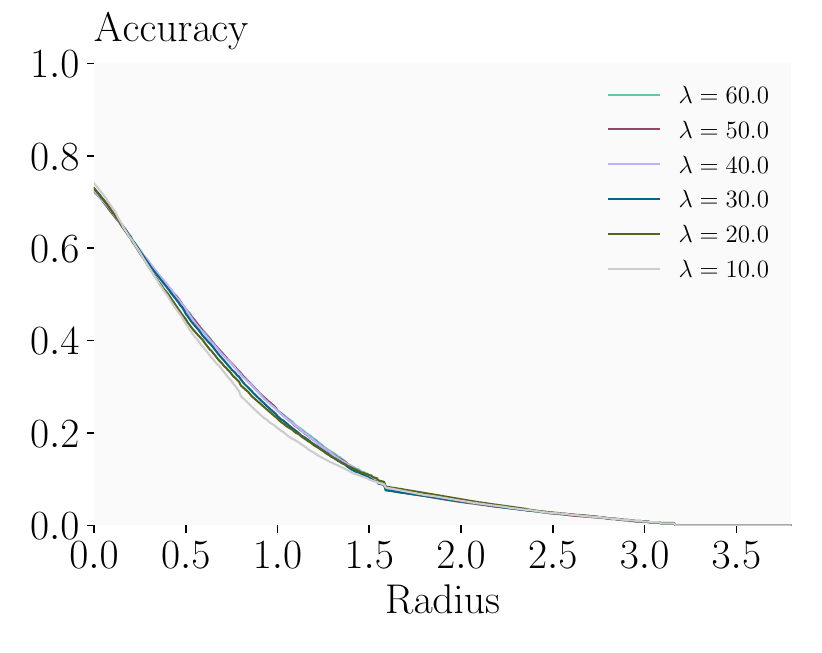}
    \caption{Accuracy–$R_{\text{final}}$ curves after \emph{1 Estimator Training + 1 Classifier Finetuning} with different $\lambda$ in the consistency loss.}
    \label{fig:abla_lbd_ft}
    \vspace{-2mm}
\end{figure*}

\begin{figure*}[]
    \centering
    \includegraphics[width=.45\linewidth]{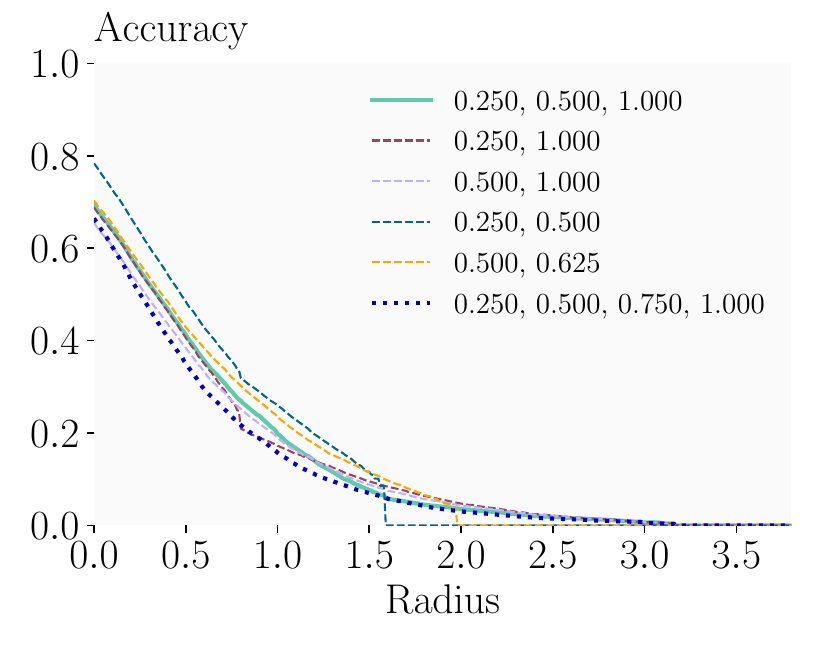}
    \caption{Certified accuracy $R_{\text{final}}$ with different $\sigma$ candidate sets.}
    \label{fig:abla_sigma_cand}
    \vspace{-2mm}
\end{figure*}

\begin{table}[]
    \caption{Certified accuracy (\%) at different radii with different $\sigma$ candidate sets. The denoiser and classifier are fixed (off-the-shelf), and only the variance estimator is trained. The best performance at each radius is highlighted in \textbf{bold}, and the worst and second worst are \textcolor{gray}{grayed}.} 
    \label{tab:abla_sigma_cand}
    \centering
    \resizebox{\linewidth}{!}{
    \begin{tabular}{cccccccccccc}
        \toprule
        $\sigma$ candidates set  & 0.00 & 0.25 & 0.50 & 0.75 & 1.00 & 1.25 & 1.50 & 1.75 & 2.00 & 2.25 & 2.50  \\ \midrule 
        \{0.25, 0.5, 1.0\}       & 69.34 & 55.25 & 41.28 & 29.01 & 19.85 & \textcolor{gray}{12.73} & \textcolor{gray}{7.62} & 4.73 & 3.54 & 2.62 & 1.83 \\
        \{0.25, 1.0\}            & 68.80 & 54.88 & 40.64 & \textcolor{gray}{26.86} & \textcolor{gray}{17.22} & 13.26 & 9.52 & 6.82 & \textbf{4.70} & \textbf{3.24} & \textbf{2.10} \\ 
        \{0.5, 1.0\}             & \textcolor{gray}{65.58} & \textcolor{gray}{51.95} & \textcolor{gray}{38.39} & 27.02 & 18.88 & 13.06 & 8.77 & 6.09 & 4.30 & 3.04 & 2.03  \\ 
        \{0.25, 0.5\}            & \textbf{78.37} & \textbf{63.52} & \textbf{48.42} & \textbf{35.61} & \textbf{25.90} & \textbf{18.41} & \textbf{11.60} & \textcolor{gray}{0.00} & \textcolor{gray}{0.00} & \textcolor{gray}{0.00} & \textcolor{gray}{0.00}  \\ 
        \{0.5, 0.625\}           & 70.29 & 56.66 & 42.90 & 32.04 & 23.49 & 16.42 & 11.48 & \textbf{7.40} & \textcolor{gray}{0.00} & \textcolor{gray}{0.00} & \textcolor{gray}{0.00}  \\ 
        \{0.25, 0.5, 0.75, 1.0\} & \textcolor{gray}{66.44} & \textcolor{gray}{50.17} & \textcolor{gray}{35.03} & \textcolor{gray}{23.55} & \textcolor{gray}{15.73} & \textcolor{gray}{10.27} & \textcolor{gray}{6.99} & \textcolor{gray}{4.54} & \textcolor{gray}{3.01} & \textcolor{gray}{2.13} &
        \textcolor{gray}{1.47}   \\ \bottomrule
    \end{tabular}
    }
\end{table}

\begin{figure*}[]
    \centering
    \begin{subfigure}[t]{.45\linewidth}
        \centering
        \includegraphics[width=.95\linewidth]{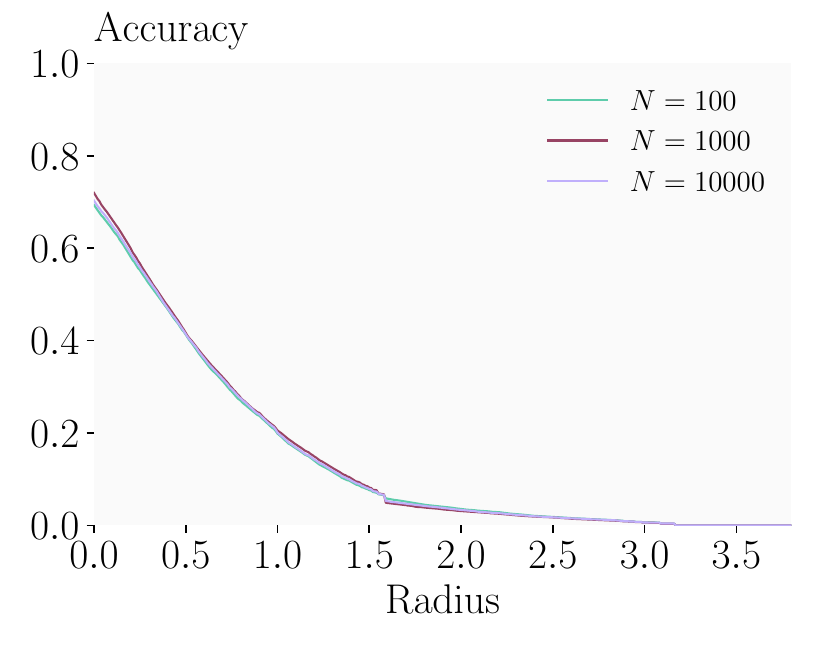}
        \caption{}
        \label{subfig:abl_map}
    \end{subfigure}
    \begin{subfigure}[t]{.45\linewidth}
        \centering
        \includegraphics[width=.95\linewidth]{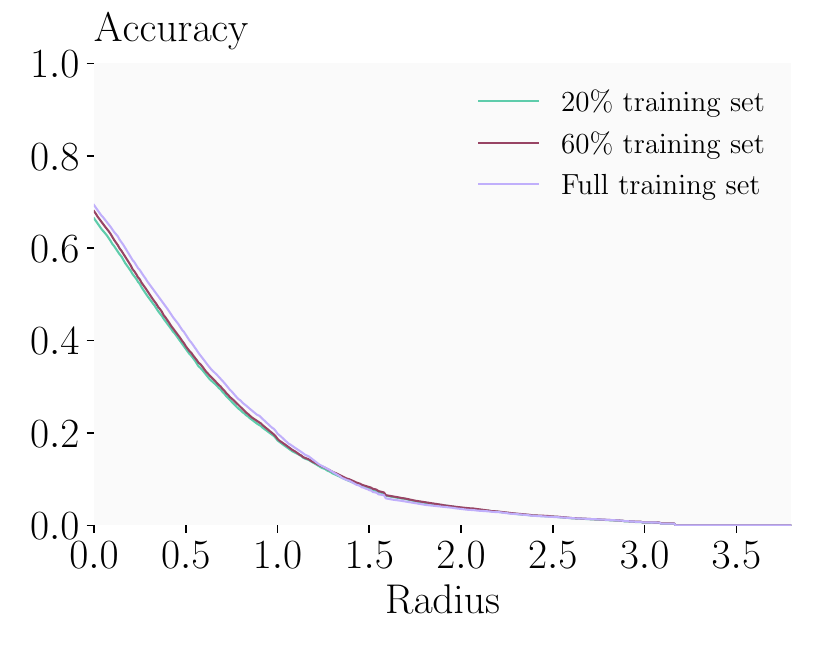}
        \caption{}
        \label{subfig:abl_subset}
    \end{subfigure}
    \caption{Study on reducing the training set construction cost. (a) Accuracy $R_{\text{final}}$ curves with different number of samples $N$ when calculating the certified radius for each $\sigma$ candidate. (b) Accuracy $R_{\text{final}}$ curves with different portion of training data used to train the variance estimator.}
    \label{fig:abla_reduce_train}
    \vspace{-2mm}    
\end{figure*}

\begin{figure*}[]
    \centering
    \includegraphics[width=.45\linewidth]{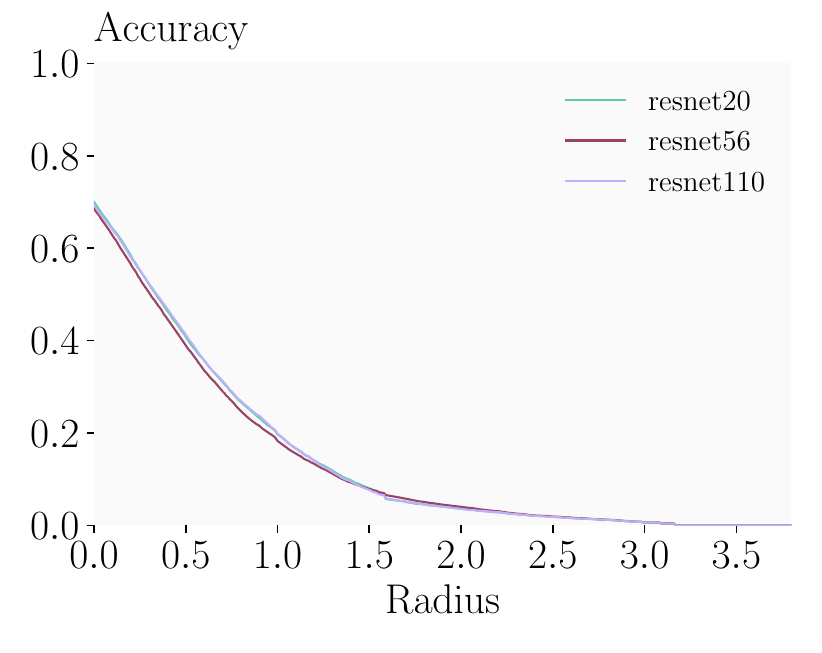}
    \caption{Accuracy $R_{\text{final}}$ curves with different $\sigma$ estimator architectures.}
    \label{fig:abla_smaller}
    \vspace{-4mm}
\end{figure*}

We employ the off-the-shelf diffusion denoiser and classifier in this study. \cref{fig:abla_lbd} illustrates the effect of $\lambda$ in the consistency loss. As $\lambda$ increases, the accuracy of variance estimation decreases and fewer samples are constrained by $R_{\sigma}$. Beyond $\lambda > 40.0$, the impact of further increases becomes negligible. With a moderate value (e.g., $\lambda = 40.0$), dual RS achieves strong performance in the medium-radius region, while incurring a slight performance drop in the small-radius region.

\subsection{Iterative Training} \label{app:finetune-rounds}

\cref{fig:abla_ft} shows training the variance estimator and finetune the classifier for different number of times. Finetuning the classifier with the estimated variances significantly improves the performance of dual RS across all radii. Moreover, re-training the variance estimator after classifier finetuning yields additional gains, particularly in the medium-radius region. \cref{fig:abla_lbd_ft} reports the Accuracy–$R_{\text{final}}$ curves after \emph{1 Estimator Training + 1 Classifier Finetuning} with different values of $\lambda$ in the consistency loss. Compared with \cref{subfig:abl_abla_loss}, which uses the off-the-shelf classifier, the influence of $\lambda$ becomes larger after finetuning. This occurs because larger $\lambda$ induces larger $R_\sigma$, thereby constraining fewer samples after finetuning, which amplifies the performance gap across different $\lambda$ values.

\subsection{Choice of $\sigma$ Candidates} \label{app:sigma-candidates}

Dual RS requires selecting a set of candidate noise variances. We conduct an ablation study to study how this choice affects certified accuracy on \cifar. In this ablation study, we use the off-the-shelf denoiser and classifier, and train the variance estimator only. In addition to the original candidate set $\{0.25, 0.5, 1.0\}$, we evaluate five alternative candidate sets: $\{0.25, 0.5\}$, $\{0.5, 0.625\}$, $\{0.25, 1.0\}$, $\{0.5, 1.0\}$, and $\{0.25, 0.5, 0.75, 1.0\}$. For each set, we use the maximum of the candidates as the estimator's global noise level, $\sigma_e$. \cref{fig:abla_sigma_cand} shows the accuracy - $R_{\text{final}}$ curves and \cref{tab:abla_sigma_cand} presents the numerical results.

Compared with $\{0.25, 0.5, 1.0\}$, using only two candidates ($\{0.25, 0.5\}$, $\{0.5, 0.625\}$, $\{0.25, 1.0\}$, or $\{0.5, 1.0\}$) leads to performance degradation at radii unfavored by the candidate set, but improving the performance at radii favored by the candidate set. Specifically, the candidate set $\{0.25, 0.5\}$ and $\{0.5, 0.625\}$ cannot achieve non-trivial accuracy at radii larger than 2.0, but achieve stronger performance at radii smaller than 2.0. The candidate set $\{0.25, 1.0\}$ leads to reduced accuracy in the medium-radii range (radii from 0.75 to 1.00), but improves at large radii (radii larger than 1.00).

 Increasing the number of candidates, e.g., using $\{0.25, 0.5, 0.75, 1.0\}$, does not improve performance over $\{0.25, 0.5, 1.0\}$, potentially due to the increased difficulty of accurately estimating the optimal $\sigma$ and obtaining a sufficiently large certified radius for the estimated $\sigma$.

\subsection{Reducing Training Cost} \label{app:reduce-training-cost}

To avoid the high cost of building training dataset, we adopt $N=100$ budget to estimate the certified radius for each $\sigma$ candidate in the main experiments. We further explore the effect of $N$ in this section. Moreover, we also explore the effect of using only a subset of the training data to train the variance estimator, which further reduces the training cost. All experiments here use off-the-shelf denoiser and classifier, on the CIFAR-10 dataset.

To see if a more accurate estimation on optimal certified radius brings performance gains, we additionally explore $N=1000$ and $N=10,000$. This increases the dataset construction cost by 10x and 100x, respectively. As shown in \cref{subfig:abl_map}, decreasing $N$ has minimal effect on performance, demonstrating that a relatively small $N$ is sufficient for training the variance estimator.

We also study whether the variance estimator can be trained only on a subset of the training data. Specifically, we randomly sample $\{20\%, 60\%\}$ of the training data and train the variance estimator solely based on the sampled training subset, thereby cutting down the training cost respectively. As shown in \cref{subfig:abl_subset}, using a significantly smaller training set for the variance estimator has minimal effect on the performance. This shows that a relatively small portion of the training data is sufficient to train a high-quality variance estimator.

In summary, both strategies substantially reduce the training cost while maintaining estimator performance. Therefore, in practice, we suggest to start with a small subset of the training data and estimate the radius based on small $N$, then progressively grow the size and the estimation quality until the performance gain diminishes.

\subsection{Architecture of the Variance Estimator} \label{app:smaller-estimator}

In the main CIFAR-10 experiments, we used a ResNet-110 estimator, which is a standard choice in training-based RS. We additionally evaluate smaller architectures (ResNet-20 and ResNet-56) while keeping the denoiser and classifier fixed. As shown in \cref{fig:abla_smaller}, using a smaller variance estimator has minimal effect on the accuracy - $R_{\text{final}}$ curves. This indicates that even though smaller estimators have lower representational capacity, they remain sufficiently expressive to approximate locally constant variance in practice.

\end{document}